\documentclass[twoside,11pt]{article}

\usepackage[preprint]{jmlr2e}
\usepackage{xcolor}
\usepackage{amsmath}

\usepackage[linesnumbered,ruled]{algorithm2e}
\usepackage{listings}
\usepackage{booktabs}
\usepackage{thmtools}
\usepackage{thm-restate}
\usepackage{mathtools}
\usepackage[inline]{enumitem}

\usepackage{tikz}
\usetikzlibrary{arrows.meta,arrows,graphs}
\usetikzlibrary{shapes,decorations,arrows,calc,positioning,automata}
\usepackage[caption=false]{subfig}
\tikzset{
  edge/.style = {
    semithick
  },
  arc/.style = {
    ->,
    semithick,
    >={[round,sep]Stealth}
  },
  bidir/.style = {
    <->,
    semithick,
    >={[round,sep]Stealth}
  }
}
\newcommand\tikzrightarrow[1][1.4em]{\tikz[baseline=-0.5ex, shorten
<=2pt, shorten >=2pt] \draw[-{Stealth[round,sep]}] (0,0) -- (#1,0);}
\newcommand\tikzleftarrow[1][1.4em]{\tikz[baseline=-0.5ex, shorten
<=2pt, shorten >=2pt] \draw[{Stealth[round,sep]}-] (0,0) -- (#1,0);}
\newcommand\tikzleftrightarrow[1][1.4em]{\tikz[baseline=-0.5ex,
  shorten  <=1pt, shorten >=1pt]
\draw[{Stealth[round,sep]}-{Stealth[round,sep]}] (0,0) -- (#1,0);}

\lstdefinelanguage{ruletable}{
  keywords={EDGES,COLORS,SETS,START,OUTPUT,AT,not,in,|},%
  literate={|}{\textbar}1,%
  comment=[l]{\#},%
  alsoletter={-,>,<},%
}
\colorlet{structure.fg}{black!80}
\lstdefinestyle{codeStyle}{
  basicstyle=\footnotesize\ttfamily,
}

\def\Cpp{{C\nolinebreak[4]\hspace{-.05em}\raisebox{.4ex}{\tiny\bf ++}}}

\def\de{\mathrm{de}}
\def\an{\mathrm{an}}
\def\pa{\mathrm{pa}}
\def\ch{\mathrm{ch}}
\def\sib{\mathrm{sib}}
\def\possan{\mathrm{possan}}
\def\possde{\mathrm{possde}}
\def\causal{\mathrm{causal}}
\def\forb{\mathrm{forb}}
\def\closure{\mathrm{closure}}
\def\dis{\mathrm{dis}}
\def\displus{\mathrm{dis}^+}

\def\zopt{Z^{\mathrm{opt}}}
\def\zoptnew{Z^{\mathrm{opt}}_{\mathrm{new}}}
\def\wopt{W^{\mathrm{opt}}}
\def\woptnew{W^{\mathrm{opt}}_{\mathrm{new}}}

\DeclareRobustCommand{\VAN}[3]{#2}

\makeatletter
\def\cleartheorem#1{%
  \expandafter\let\csname#1\endcsname\relax
  \expandafter\let\csname c@#1\endcsname\relax
}
\makeatother
\cleartheorem{theorem}
\declaretheorem[name=Theorem]{theorem}

\cleartheorem{example}
\declaretheorem[name=Example]{example}

\cleartheorem{proposition}

\usepackage{lastpage}
\jmlrheading{xx}{2025}{1-\pageref{LastPage}}{x/xx; Revised
x/xx}{x/xx}{21-0000}{Marcel Wienöbst, Sebastian Weichwald and Leonard Henckel}

\ShortHeadings{Wien\"obst, Weichwald and Henckel}{Linear-Time Primitives for Causal Inference Algorithms}
\firstpageno{1}

\usepackage{tipa}

\begin{document}

\title{Linear-Time Primitives for Algorithm Development in Graphical Causal Inference}

\author{\name Marcel Wien\"obst \email m.wienoebst@uni-luebeck.de \\
  \addr Institute for Theoretical Computer Science\\
  University of Lübeck, Germany
  \AND
  \name Sebastian Weichwald \email sweichwald@math.ku.dk \\
  \addr Department of Mathematical Sciences\\
  University of Copenhagen, Denmark
  \AND
  \name Leonard Henckel \email leonard.henckel@ucd.ie \\
  \addr School of Mathematics and Statistics\\
University College Dublin, Ireland}

\editor{My editor}

\maketitle

\begin{abstract}%
We introduce CIfly, a framework for efficient algorithmic primitives in graphical causal inference that isolates reachability as a reusable core operation. It builds on the insight that many causal reasoning tasks can be reduced to reachability in purpose-built state-space graphs that can be constructed on the fly during traversal. We formalize a rule table schema for specifying such algorithms and prove they run in linear time. We establish CIfly as a more efficient alternative to the common primitives moralization and latent projection, which we show are computationally equivalent to Boolean matrix multiplication. Our open-source Rust implementation parses rule table text files and runs the specified CIfly algorithms providing high-performance execution accessible from Python and R. We demonstrate CIfly’s utility by re-implementing a range of established causal inference tasks within the framework and by developing new algorithms for instrumental variables. These contributions position CIfly as a flexible and scalable backbone for graphical causal inference, guiding algorithm development and enabling easy and efficient deployment.
\end{abstract}


\section{Introduction}
Causal graphical models that encode covariates as nodes and the qualitative
causal relationships between them as edges are a popular framework for rigorous
causal and probabilistic reasoning \citep{pearl2009causality}. For a variety of
theoretical and practical reasons, the causal inference literature has
developed many types of graphical models encoding different levels of causal
information, such as directed acyclic graphs (DAGs), completed partially
directed acyclic graphs (CPDAGs), acyclic direct mixed graphs (ADMGs), maximal
ancestral graphs (MAGs), and partial ancestral graphs (PAGs). This variety of
graph types can be challenging but causal graphical models are nonetheless
attractive as they allow for mathematically sound and, importantly, automated
causal and probabilistic reasoning. A fundamental probabilistic task, for
example, is to find a separating set between two variables, that is, a set of
random variables such that the two are conditionally independent
\citep{geiger1989dseparation}. A popular causal task is to identify a
functional in the observational density that estimates some target causal
quantity \citep[for example,][]{huang2006identifiability}. This task is also often
studied under restrictions on the functionals, such as, that they must
correspond to an adjustment formula \citep{perkovic2018complete}, an
instrumental variable model \citep{brito2002generalized}, a front-door model
\citep{pearl1995causal} or a g-formula \citep{robins1986new,
perkovic2020identifying}. Other common tasks are to verify whether a candidate
functional is valid \citep{perkovic2018complete} or to find the most
statistically efficient valid functional in some class
\citep{rotnitzky2020efficient,witte2020efficient, henckel2019graphical}. 

For truly automated and scalable causal reasoning, we need efficient algorithms
with broadly accessible implementations. This is particularly important in
applications with large graphs \citep{maathuis2010predicting} or complex
simulation studies \citep{heinze2018causal}. In response, a broad literature on
algorithms in causal inference has developed \citep[for example,][]{shachter1998bayes,
van2019separators}. However, due to the diversity of causal graph types, the
variety of causal reasoning tasks, and the continual development of new
methodology \citep[for example,][]{guo2023variable, henckel2024graphical}, these
efforts remain ongoing.  Moreover, efficient implementations are not always
available \citep{peters2015structural,jeong2022frontdoor}, or only in few
programming languages \citep{kalisch2012causal,textor2016robust}, which limits
their utility for the growing and academically diverse community interested in
causal inference. This stands in contrast to the overwhelming success of
efficient implementations for many machine learning algorithms built upon
linear algebra primitives, which are broadly accessible across programming
languages.

Graphical causal algorithms can often be built upon reachability-based
primitives. In its standard form, reachability is the problem of finding all
nodes that can be reached from some starting set of nodes in a graph via
directed paths. It can be solved in linear time in the input size, that is, given
a graph with $p$ nodes and $m$ edges, in $O(p+m)$, by breadth-first (BFS) or
depth-first search (DFS). Reachability can not only be used for tasks such as
finding descendants in a DAG, but also more generally by applying it to
problem-specific graph constructions. In the causal inference literature, this
approach was first used by \citet{shachter1998bayes} in the Bayes-Ball
algorithm for verifying d-separation statements in DAGs. Bayes-Ball can be seen
as a reachability algorithm in a new, not necessarily acyclic, directed graph,
where each node represents a node-edge tuple from the original DAG and the
edges are added based on rules that depend on the conditioning set (see Figure
\ref{figure:dsep:statespace}). The connection between Bayes-Ball and
reachability was made explicit by \citet{van2020algorithmics} who also noted
that, by modifying the edge rules, it is possible to solve other tasks via
reachability in this node-edge tuple graph, for example, finding the possible
descendants in a CPDAG. Such modified edge-rule reachability algorithms are
also used as primitives in the first linear-time algorithms for finding a
minimal d-separating set \citep{van2020finding} or a (minimal) front-door set
\citep{wienobst2024linear}. \citet{wienobst2024linear} also noted that it is
possible to succinctly encode these reachability algorithms with a rule table.
Finally, \citet{henckel2024adjustment} used a similar approach, but relaxed the
requirement that nodes in the problem-specific graph be node-edge tuples from
the original graph, to develop a linear-time algorithm that, for a fixed
treatment, verifies the validity of an adjustment set for all possible outcomes
simultaneously.

In this paper, we build on these developments to propose
CIfly~(\textipa{/"saI""flaI/}), a formal and flexible framework for developing
linear-time algorithmic primitives, called CIfly algorithms. CIfly formalizes
the idea of creating a task-specific state-space graph on the fly while running
a standard reachability algorithm. We prove that CIfly algorithms have linear
run-time complexity and develop a schema to specify them with rule tables that
can be encoded as simple text files. We also provide the \texttt{cifly} package,
written in Rust, that parses such rule table text files and enables the
efficient execution of the corresponding CIfly algorithms. This functionality
is made available in Python and R via the \texttt{ciflypy} and \texttt{ciflyr}
packages.

We showcase CIfly by using it to re-implement the algorithms for various
important causal inference tasks, thus demonstrating its broad applicability.
We establish the competitive performance of CIfly algorithms by benchmarking
against existing implementations in \texttt{pcalg} by
\citet{kalisch2012causal}, \texttt{DAGitty} by \citet{textor2016robust}, and
\texttt{gadjid} by \citet{henckel2024adjustment}. We also use CIfly to
implement three new algorithms for instrumental sets: The first sound and
complete algorithm for verifying whether a candidate set is a valid conditional
instrumental set, the first linear-time algorithm for identifying the
statistically most efficient conditional instrumental set, and the first sound
and complete algorithm for finding a valid conditional instrumental set. The
rule tables and implementations for these and further algorithms are available
on the CIfly website at \url{cifly.pages.dev}. Finally, we contrast CIfly to
two other algorithmic primitives in graphical causal inference, namely
moralization \citep{cowell2007probabilistic} and latent projection
\citep{richardson2003markov}. We show that both are computationally equivalent
to Boolean matrix multiplication and, hence, can currently \emph{not} be solved
in time quadratic in the number of nodes. This disproves contrary claims about
moralization in the literature. Neither moralization nor latent projection are therefore
suitable primitives for algorithms with time complexity linear in the input
size.

CIfly facilitates and transforms efficient algorithm development and deployment
in causal inference. First, the CIfly framework provides guidance for
developing causal algorithms based on reachability primitives, while
guaranteeing linear run time. Second, the rule tables provide a succinct,
language-agnostic way to specify and store CIfly algorithms. Third, our
open-source software packages makes it possible to efficiently deploy any
CIfly algorithm from its rule table to both Python and R. This obviates the
time-consuming and error-prone task of re-implementing algorithms across
languages. Combined, CIfly provides an efficiently implemented, easy-to-use and
extensible framework for linear-time primitives as well as guidance for
developing new algorithms. 

\section{Preliminaries}

In this paper, we study algorithms for graphs where nodes represent random
variables and edges encode different types of causal relations. We begin with
an overview of the main graphical and algorithmic notions. For space reasons,
and because they are not required for the results in this paper, we do not
formally introduce the causal models corresponding to these graphical objects
but refer the reader to the wider causal inference literature for these
definitions \citep[for example,][]{pearl2009causality, perkovic2018complete,
henckel2024graphical}.

\vspace{.1cm}

\noindent\emph{Edge-types, graphs and directed graphs:} 
To accommodate the many graph types used within the causal inference
literature, we use the following broad definition of a graph. Let
$\mathcal{E}=\{\epsilon_1,\dots, \epsilon_k\}$ denote a finite set of distinct
edge-types, for example, $\mathcal{E}=\{\tikzrightarrow,\tikzleftrightarrow\}$. Every
edge-type is either ordered, such as \ $\tikzrightarrow$, or unordered, such as
$\tikzleftrightarrow$. A graph $G$ with edge-types $\mathcal{E}$ consists of a
set of \emph{nodes} $V$ and a $k$-tuple of edge sets $(E_1, \dots, E_k)$, where
$E_i$ corresponds to the edge-type $\epsilon_i$ and satisfies either $E_i
\subseteq V^2$ if $\epsilon_i$ is ordered or $E_i \subseteq \binom{V}{2}$ if
$\epsilon_i$ is unordered. A \emph{directed} graph, for example, has edge-type
set $\{\tikzrightarrow\}$, that is, it only contains edges of the form $u
\tikzrightarrow v$, and can thus be defined as $G = (V, (E_1))$ with $E_1
\subseteq V^2$. For graphs with a single set of edges, we write $G = (V, E)$
for simplicity. We use $\mathcal{G}_{\mathcal{E}}$ to denote the set of all
graphs with edge-type set $\mathcal{E}$. Given a graph $G$, we use $p$ to
denote the number of nodes in $G$ and $m$ to denote the total number of edges.

\vspace{.1cm}

\noindent\emph{Adjacency, walks and paths:}
A \emph{walk} $w$ in $G$ is defined by a sequence of nodes $v_1, \dots,v_\ell$
and a corresponding sequence of edges $e_1, e_2, \dots, e_{\ell-1}$ such that
$e_i$ is an edge between $v_i$ and $v_{i+1}$ for all $i \in \{1, \dots,
\ell-1\}$.  Nodes with an edge in common are called \emph{adjacent}. An edge
connecting $u$ and $v$ is \emph{incident} to these nodes.  A \emph{path} is a
walk that contains distinct nodes. The first node $v_1$ and the last node
$v_\ell$ on a walk $w$ are called endpoints of $w$. A walk (or path) from a set
$X$ to a set $Y$ is a walk (or path) from some $x \in X$ to some $y \in Y$. A
walk (or path) from a set $X$ to a set $Y$ is proper if only the first node is
in $X$. 

\vspace{.1cm}

\noindent\emph{Neighbor-types and sets:}
Given a set of edge-types $\mathcal{E}$, we use $\mathcal{N}_{\mathcal{E}}$ to
denote the set of neighbor-types in a graph $G \in \mathcal{G}_{\mathcal{E}}$.
The set $\mathcal{N}_{\mathcal{E}}$ contains one element for each unordered
edge-type in $\mathcal{E}$ and two for each ordered one. Consider, for example,
a node $v$ in a graph $G = (V,(E_1,E_2))$ where $E_1$ contains ordered edges of
the form $\tikzrightarrow$ and $E_2$ unordered edges of the form $-$. Every
node $u$ adjacent to $v$ in $G$ satisfies $u \tikzrightarrow v$, $u
\tikzleftarrow v$ or $u - v$ in which case $u$ is called a parent, a child or a
sibling of $v$, respectively. We therefore, have three types of neighboring
nodes, that is,
$\mathcal{N}_{\{\tikzrightarrow,-\}}=\{\tikzrightarrow,\tikzleftarrow,-\}$.
Given a neighbor type $n$ and a node $v$, we use $N_n(v)$ to denote the set of
neighbors of $v$ that have type $n$. 

\vspace{.1cm}

\noindent\emph{Directed acyclic graphs (DAGs) and acyclic directed mixed graphs (ADMGs):}
A walk $w$ from $x$ to $y$ is called directed towards $y$ if all edges on $w$
are of the form $v_i \tikzrightarrow v_{i+1}$ and point towards $y$. A directed
walk from $x$ to $y$ and the edge $y \tikzrightarrow x$ forms a directed cycle.
A directed graph without a directed cycle is called \emph{acyclic} and we call
such graphs directed acyclic graphs (DAGs). A graph $G = (V, (E_1, E_2))$ is
called an acyclic directed mixed graph (ADMG) if $E_1$ is a set of directed
edges $\tikzrightarrow$, $E_2$ is a set of bidirected edges
$\tikzleftrightarrow$, and the subgraph $G_{\mathrm{dir}} = (V, (E_1))$ is
acyclic. In causal inference, DAGs and ADMGs are commonly used to encode
qualitative causal information but they also represent conditional
independencies \citep{spirtes2000causation}. The conditional independence
statements can be read off from these graphs using the d-separation criterion. 

\vspace{.1cm}

\noindent\emph{Colliders, non-colliders and d-separation:}
Given a walk $w$ in a DAG $G=(V,E)$, we call a node $v$ a collider on $w$ if
the two edges on $w$ incident to $v$ are of the form $\tikzrightarrow v
\tikzleftarrow$. We call $v$ a non-collider on $w$ if it is neither an endpoint
node nor a collider on $w$. Consider two disjoint node sets $X$ and $Z$ in a
DAG $G = (V, E)$. Then, $Y \subseteq V \setminus (X \cup Z)$ is
\emph{d-connected} to $X$ in $G$ if there exist $x \in X, y \in Y$ and a walk
$w$ from $x$ to $y$ such that 
\begin{enumerate*}[label=(\roman*)]
  \item\label{condition:collider} for any collider $v$ on $w$, it holds that $v \in Z$ and
  \item\label{condition:non-collider} for any non-collider $v$ on $w$, it holds that $v \not\in Z$.
\end{enumerate*}
This is commonly denoted as $X \not\perp_{G} Y \mid Z$. Conversely, if $X$ and
$Y$ are not d-connected, they are called \emph{d-separated}, written $X
\perp_{G} Y \mid Z$. A walk that satisfies Conditions \ref{condition:collider}
and \ref{condition:non-collider} is called \emph{open} and walk that does not
\emph{closed}. This definition of d-separation is equivalent to the more
popular one using paths \citep{shachter1998bayes,pearl2009causality}.
For ADMGs, d-separation is defined analogously, with a collider on a walk being
a node $v$ with incident edges having arrowtips pointing towards $v$, that is,
$\tikzrightarrow v \tikzleftarrow,\tikzrightarrow v
\tikzleftrightarrow,\tikzleftrightarrow v \tikzleftarrow$ or
$\tikzleftrightarrow v \tikzleftrightarrow$.

\vspace{.1cm}

\noindent\emph{Completed partially directed acyclic graphs (CPDAGs) and possibly directed paths:}
Several DAGs can encode the same d-separation statements. Such DAGs form a
Markov equivalence class which can be described uniquely by a completed
partially directed acyclic graph \citep{meek1995causal}. A CPDAG contains
directed edges $\tikzrightarrow$ and undirected edges $-$. An edge $v_1
\tikzrightarrow v_2$ in a CPDAG $C$ encodes that all DAGs in the Markov
equivalence class contain the edge $v_1 \tikzrightarrow v_2$, whereas $v_1 -
v_2$ encodes that some, but not all, contain the edge $v_1 \tikzleftarrow v_2$
instead. A walk $w$ from $x$ to $y$ is called possibly directed if all directed
edges on $w$ point towards $y$. We call a walk that is not possibly directed, non-causal.

\vspace{.1cm}

\noindent\emph{Ancestral relations and forbidden nodes:}
Consider a DAG, ADMG or CPDAG $G$. If $x \tikzrightarrow y$ then $x$ is a
parent of $y$ and $y$ is a child of $x$. If $x \tikzleftrightarrow y$ then $x$
is a sibling of $y$. If there is a directed path from $x$ to $y$ then $y$ is a
descendant of $x$ and $x$ is an ancestor of $y$. If there is a possibly
directed path from $x$ to $y$ then $y$ is a possible descendant of $x$ and $x$
is a possible ancestor of $y$. If this path does not contain nodes in $W$,
then we say the $y$ is a proper (possible) descendant of $y$ relative to $W$
and $x$ is a proper (possible) ancestor of $y$ relative to $W$. We denote the
corresponding sets with $\pa_G(x)$, $\ch_G(x)$, $\sib_G(x)$, $\de_{G,W}(x)$,
$\an_{G,W}(x)$, $\possde_{G,W}(x)$ and $\possan_{G,W}(x)$, where we drop $W$ if
it is the empty set. For a set $X$, $\pa_G(X)\coloneq\bigcup_{x \in X}
\pa_G(x)$ with analogous definitions for the other sets. Let $\causal_G(X, Y)
\coloneq \possan_{G,X}(Y) \cap \possde_G(X)$ be the set of causal nodes
relative to $(X,Y)$ in $G$ and $\forb_G(X,Y) \coloneq \possde_G(\causal_G(X,
Y)) \cup X$ the set of forbidden nodes. We drop the $\mathrm{poss}$ prefixes
when considering a DAG or an ADMG. 

\vspace{.1cm}

\noindent\emph{Causal Primitives:}
A causal primitive is a task that, for an input consisting of a graph
$G=(V,(E_1,\dots,E_k))$ and a sequence of sets $L = (L_1, L_2 \dots,
L_{\ell})$, produces a set $R \subseteq V$ as output. An example is the problem of finding all nodes d-connected to a set of nodes $X$
in a DAG $G$ given a set $Z$. The input here is $G$ and $L=(X, Z)$ and
the output $R$ the set of d-connected nodes. Many tasks
in causal inference do not immediately fit into this form, for example, because
the output is a Boolean instead of a set (is there a $Z$ d-separating $X$ and $Y$ in $G$?), but these tasks can often be solved using causal primitives as subtasks with minor post-processing, such as
checking a set membership (see Figure \ref{figure:admgcode}). 

\vspace{.1cm}

\noindent\emph{Reachability:}
In a directed graph $G_{\mathrm{dir}}=(V_{\mathrm{dir}},E_{\mathrm{dir}})$, a
node $y$ is \emph{reachable} from node $x$ in $G$ if there exists a directed
path from $x$ to $y$ in $G$. Given two sets of nodes $X, Y \subseteq V_{\mathrm{dir}}$, we say
that $Y$ is reachable from $X$ if there exist $x \in X$ and $y \in Y$ such that
$y$ is reachable from $x$. The computational problem \textsc{reach} takes as
\emph{input} a graph $G_{\mathrm{dir}} = (V_{\mathrm{dir}}, E_{\mathrm{dir}})$
and a set $S \subseteq V_{\mathrm{dir}}$ and returns as \emph{output} the set
$T \subseteq V_{\mathrm{dir}}$ of nodes reachable from $S$ in
$G_{\mathrm{dir}}$. Classical algorithms such as breadth-first search (BFS) and
depth-first search (DFS) initiated from $S$ \citep{cormen2022introduction}
solve \textsc{reach} in linear time, that is, in time $O(p+m)$. To decide whether
a set $Y$ is reachable from $X$ in linear time, one can solve \textsc{reach} for
$S = X$ and check if the intersection $Y \cap T$ is non-empty. 

\section{CIfly: Reachability in State-Space Graphs Created on the Fly}
\label{section:framework}
In this section, we propose a framework for designing linear-time algorithms in
causal inference called CIfly. It relies on reductions of graphical causal
inference tasks to \textsc{reach} instances. We first define the general
framework and show that any CIfly reduction can be solved by a linear-time
algorithm. We then introduce rule tables as purpose-built notation for
succinctly encoding a CIfly algorithm. Finally, we discuss our software
implementation that enables users to deploy CIfly algorithms specified by rule
tables to R or Python.

\begin{figure}
\centering
  \begin{tikzpicture}
    \node (leftcol)   at (0,-.5) {\textsf{Causal Primitive}};
    \node (rightcol)  at (6,-.5) {\textsf{Reachability Task}};
    \node (upperrow)  at (-3.5,-2) {\textsf{Input}};
    \node (lowerrow)  at (-3.64,-4) {\textsf{Output}};
    \node (causalin)  at (0,-2) {$G, (L_1, \dots, L_{\ell})$};
    \node (causalout) at (0,-4) {$R$};
    \node (reachin)   at (6,-2) {$G_{\mathrm{dir}}, S$};
    \node (reachout)  at (6,-4) {$T$};

    \draw[arc] (2,-2) -- (4,-2) node[midway, above]{$\mathcal{I}$};
    \draw[arc] (6,-2.5) -- (6,-3.5) node[midway, right=0.1cm]{\textsc{reach}};
    \draw[arc] (4,-4) -- (2,-4) node[midway, above]{$\mathcal{O}$};
    \draw[arc, dotted] (0,-2.5) -- (0,-3.5);
  \end{tikzpicture}
\caption{An illustration of the general principle underlying causal-to-reach reductions.}\label{figure:causaltoreach}
\end{figure}
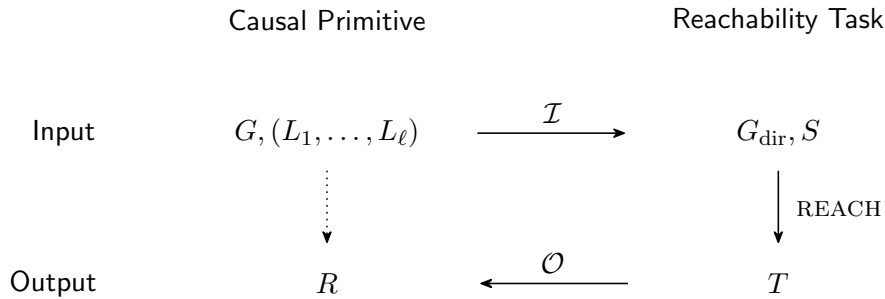

\subsection{The CIfly Framework}
Suppose we are interested in developing an efficient algorithm for solving a
causal inference task. We assume that the task can be solved through subtasks
that can be expressed as causal primitives, that is, tasks that take as input a
graph $G$ and an $\ell$-tuple of node sets in $G$ and produce a set of nodes as
output. The principle underlying CIfly is to solve causal primitives by mapping
the input to a corresponding \textsc{reach} instance, solving this task by
standard algorithms and, afterwards, mapping the output back to the original
domain. Hence, it can be seen as a framework for reductions from causal
primitives to reachability. We formalize such reductions in the following
definition. 

\begin{definition}[Causal-to-reach reduction]\label{definition:causaltoreach}
  Consider a set of edge-types $\mathcal{E}$ and an integer $\ell$. An
  $(\mathcal{E},\ell)$-causal-to-reach reduction consists of two maps
  $\mathcal{I}$ and $\mathcal{O}$:
  \begin{itemize}
    \item $\mathcal{I}$ is a map assigning to each pair, consisting of a graph $G
      \in \mathcal{G}_{\mathcal{E}}$ and $\ell$-tuple of node sets
      $(L_1,\dots,L_\ell)$ in $G$, a reachability input $(G_{\mathrm{dir}},S)$.
    \item $\mathcal{O}$ is a map  from a set $T$ of nodes in $G_{\mathrm{dir}}$,
      specifically the output of \textsc{reach} applied to
      $(G_{\mathrm{dir}},S)$, to a set of nodes in $G$. 
  \end{itemize}
\end{definition}

The map $\mathcal{I}$ constructs the \textsc{reach} input, whereas the map
$\mathcal{O}$ maps the resulting \textsc{reach} output back to a set of nodes
of $G$. Overall, this yields an algorithm consisting of three steps: the
mapping to a \textsc{reach} input, the reachability algorithm itself and the
translation back to the causal domain. This approach is illustrated in
Figure~\ref{figure:causaltoreach}. While
Definition~\ref{definition:causaltoreach} formalizes the general principle of
CIfly, it is too general to be of much use. In particular, while \textsc{reach}
can be solved in linear time, this may not hold for the reduction, meaning the
overall procedure may not yield a linear-time algorithm for solving the task at
hand. Furthermore, it is unclear how to encode an arbitrary
$(\mathcal{E},\ell)$-causal-to-reach reduction, which is important from both a
methodological perspective and for software development. We now propose a
subclass of $(\mathcal{E},\ell)$-causal-to-reach reductions that avoids these
two problems.

\begin{definition}[CIfly reduction]\label{definition:cifly} 
  Let $\mathcal{E}$ be a set of edge types and $\ell$ an integer. An
  $(\mathcal{E}, \ell)$-\emph{CIfly reduction} is defined by the tuple $(C,
  \mathcal{S}, \mathcal{T}, \phi)$, where
  \begin{enumerate} 
    \item $C$ is a finite set whose elements we call \emph{colors},
    \item $\mathcal{S}$ is a subset of $\{1, 2, \dots, \ell\} \times
      \mathcal{N}_{\mathcal{E}} \times C$, 
    \item $\mathcal{T}$ is a subset of $\mathcal{N}_{\mathcal{E}} \times C$,
      and 
    \item $\phi$ is a  function mapping from $\left(\{0, 1\}^{\ell} \times 
      \mathcal{N}_{\mathcal{E}} \times C \right)\times \left(\{0, 1\}^{\ell} \times 
      \mathcal{N}_{\mathcal{E}} \times C \right)$ to $\{0, 1\}$. 
  \end{enumerate}
  The corresponding map $\mathcal{I}_{\text{CIfly}}$ maps any $G \in
  \mathcal{G}_{\mathcal{E}}$ with node set $V$ and any $(L_1,\dots,L_\ell)
  \subseteq V^\ell$ to the \textsc{reach} input $(G_{\mathrm{dir}} =
  (V_{\mathrm{dir}}, E_{\mathrm{dir}}), S)$, where
  \begin{enumerate}
    \item $V_{\mathrm{dir}} = V \times \mathcal{N}_{\mathcal{E}} \times C$, 
    \item 
      $\begin{aligned}[t]
  E_{\mathrm{dir}} = \{ (v_1, &n_1, c_1)\tikzrightarrow (v_2, n_2, c_2)
    \mid v_2 \in N_{n_2}(v_1) \text{ and } \\ 
  &\phi(v_1 \in L_1, \dots, v_1 \in L_{\ell}, v_2 \in L_1, \dots, v_2 \in
    L_{\ell}, n_1, n_2, c_1, c_2) = 1\}, \text{ and}
      \end{aligned}$
    \item $S = \{(v, n, c) \mid \text{there exists } (i, n, c) \in \mathcal{S}
      \text{ such that } v \in L_i\}$.
  \end{enumerate}
   The corresponding map $\mathcal{O}_{\text{CIfly}}$ maps the output $T$ of
   the resulting \textsc{reach} input to the set $R = \{v \mid \text{there
   exists } (v, n, c) \in T \text{ such that } (n, c) \in \mathcal{T} \}$. 
\end{definition}

To differentiate the causal primitive input graph $G$ from the corresponding
\textsc{reach} input graph $G_{\mathrm{dir}}$ we call the latter state-space
graph. We also call the nodes and edges of $G_{\mathrm{dir}}$ states and
transitions. Each element of the CIfly reduction $(C, \mathcal{S}, \mathcal{T},
\phi)$ has a clear role in the construction of these states and transitions.
The set $C$ controls the number of states. The set $S$ provides the rules for
how to construct the set of starting states for \textsc{reach}. The Boolean
function $\phi$ is used to determine which transitions exist in
$G_{\mathrm{dir}}$. Finally, $R$ gives the rules how to map the resulting
\textsc{reach} output consisting of states to the correct causal primitive
output. In a concrete implementation, we do not preconstruct
$G_{\mathrm{dir}}$, but generate neighboring states and check transitions on
the fly. For this reason, we call our framework causal inference on the fly or,
for short, \emph{CIfly}. In Algorithm \ref{algorithm:ciflydfs}, we provide a
concrete algorithm illustrating this approach. CIfly reductions can be used to
efficiently solve many causal primitives. The intuitive reason why, is that
causal primitives typically involve statements about the existence or absence
of certain paths or walks in the original graph $G$, with specific conditions
concerning the preceding and next edge type. CIfly translates such problems to
\textsc{reach} by constructing the state-space graph $G_{\mathrm{dir}}$. Here
nodes are node-edge-color tuples from the original graph and therefore directed
paths correspond to walks in the original graph $G$, with the edge-types and
colors encoding additional information about these walks. By constructing an
appropriate $\phi$ function we can ensure that these walks satisfy the
properties we require of them. 

The choice function $\phi$ is therefore at the core of any CIfly reduction. It
is often complex and may be represented in different ways, which we discuss in
detail in the next subsection. For now, the observation suffices that there
are, for concrete $(\mathcal{E}, \ell)$ and a finite number of colors, only a
constant amount of such functions, which can be evaluated in constant time. By
this fact and since $G_{\mathrm{dir}}$, by construction, has size linear in the
original input graph $G$, CIfly reductions can be implemented to run in linear
time. Before proving this, we illustrate a concrete CIfly reduction for a basic
causal primitive.

\begin{figure}[t]
  \centering
  \begin{tikzpicture}
    \node (linput) at (0.75,3.0) {\textsf{\large Causal-primitive input:}};
    \node (x) at (0,1.75) {$v_2$};
    \node (m) at ($(x)+(1.5,0)$) {$v_3$};
    \node (y) at ($(x)+(3,0)$) {$v_4$};
    \node (z) at ($(x)+(-1.5,0)$) {$v_1$};
    \node (l) at ($(x)+(0,-1.5)$) {$v_5$};

    \path[arc] (m) edge (x);
    \path[arc] (m) edge (y);
    \path[arc] (z) edge (x);
    \path[arc] (x) edge (l);

    \node (ll) at (-0.25,-0.85) {$L = (\{v_1\}, \{v_5\})$};
    \node (ltransition) at (7.75,3.0) {\textsf{\large \textsc{reach} input:}};
    \node (xu) at ($(x)+(7,0)$) {$v_2$};
    \draw [draw=black, rounded corners] ($(xu)+(-2.2,-.5)$) rectangle ($(xu)+(5,.8)$);
    \node (mu) at ($(xu)+(1.5,0)$) {$v_3$};
    \node (yu) at ($(xu)+(3,0)$) {$v_4$};
    \node (zu) at ($(xu)+(-1.5,0)$)   {$v_1$};
    \node (lu) at ($(xu)+(4.5,-0)$)   {$v_5$};
    \node (edgestate) at ($(xu)+(-1.85,.5)$) {\small $\tikzleftarrow$};

    \node (xl) at (7,0.25) {$v_2$};
    \draw [draw=black, rounded corners] ($(xl)+(-2.2,-.5)$) rectangle ($(xl)+(5,0.8)$);
    \node (ml) at ($(xl)+(1.5,0)$) {$v_3$};
    \node (yl) at ($(xl)+(3,0)$) {$v_4$};
    \node (zl) at ($(xl)+(-1.5,0)$)   {$v_1$};
    \node (ll) at ($(xl)+(4.5,0)$)   {$v_5$};
    \node (edgestate) at ($(xl)+(-1.85,.5)$) {\small $\tikzrightarrow$};

    \path[arc]            (xu) edge (zu);
    \path[arc, color=red] (xu) edge (mu);
    \path[arc]            (yu) edge (mu);

    \path[arc]                       (zl) edge (xl);
    \path[arc]                       (ml) edge (xl);
    \path[arc]                       (ml) edge (yl);
    \path[arc,bend right=15, color=red] (xl) edge (ll);

    \path[arc,bend left=55]            (xu) edge (ll);
    \path[arc]            			  (mu) edge (xl);
    \path[arc, color=red] (ll) edge (xu);
    \path[arc, color=red]             (zu) edge (xl);
    \path[arc, color=red]             (mu) edge (yl);

    \node (ll) at (6.15,-.85) {$S = \{v_1,\tikzleftarrow\}$};
  \end{tikzpicture}
  \caption{Causal-primitive input (left) to the CIfly reduction from
    Example \ref{example:bayesball} for finding all d-connected nodes in
    DAGs and the corresponding \textsc{reach} input (right). We split the
    state-space graph into states with neighbor-type $\tikzleftarrow$
    (top rectangle) and states with neighbor-type $\tikzrightarrow$
    (bottom rectangle) to reduce clutter.}
  \label{figure:dsep:statespace}
\end{figure}
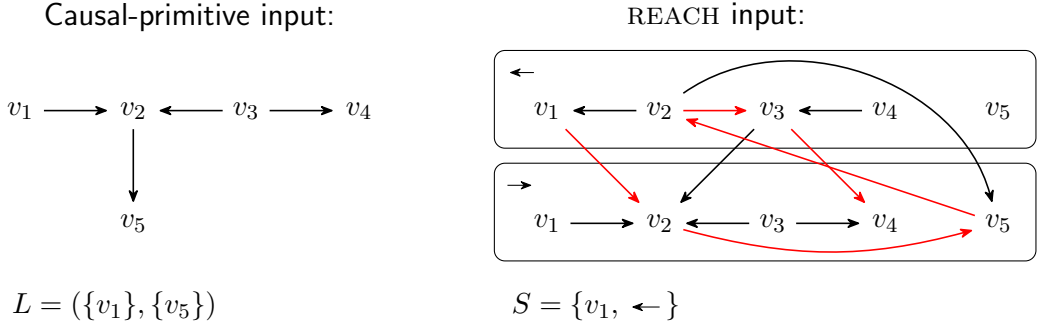

\begin{example}
  \label{example:bayesball}
  Suppose we are given a DAG $G$ and a pair of node sets $L=(X,Z)$ and are
  asked to find all nodes $y$ such that $X$ is d-connected to $y$ given $Z$.
  Thus, we have $\mathcal{E} = \{ \tikzrightarrow \}$, which implies
  $\mathcal{N}_{\mathcal{E}} = \{\tikzrightarrow, \tikzleftarrow\}$, and $\ell
  = 2$. We can solve this task using the $(\mathcal{E}, \ell)$-CIfly reduction
  with input $(G,(X,Z))$ specified as follows:
  \begin{itemize}
    \item $C= \{ \cdot \}$. Here, distinctions based on $C$ are not needed and
      we set a default color. For each node $v$ in $G$ the state-space graph
      therefore contains a state of form $(v,\tikzleftarrow,\cdot)$ and another
      of the form $(v,\tikzrightarrow,\cdot)$. For ease of presentation, we
      omit the color. 
    \item $\mathcal{S} = \{(1, \tikzleftarrow)\}$. The reachability algorithm
      thus starts at the set of nodes $\{(v, \tikzleftarrow) \mid v \in X
      \}$, since $X$ is the first set in $L$. 
    \item $\mathcal{T} = \{(\tikzleftarrow), (\tikzrightarrow)\}$. By including
      both possible neighbor types, the output set of nodes thus consists of
      all $v$ reached in a state $(v, \tikzleftarrow)$ or $(v,
      \tikzrightarrow)$. 
    \item 
      $\begin{aligned}[t]
  \phi&(v_1 \in X, v_1 \in Z, v_2 \in X, v_2 \in Z, n_1, n_2) \\ 
      &= \big(n_1 = \tikzrightarrow \land n_2 = \tikzleftarrow
      \land v_1 \in Z\big) \lor \big((n_1 \neq \tikzrightarrow
      \lor n_2 \neq \tikzleftarrow) \land (v_1 \not\in
      Z)\big).
      \end{aligned}$\hfill\mbox{}

      Here, we express $\phi$ directly as a logical formula. In the next subsection we introduce a more convenient representation.
  \end{itemize}
  By construction, a path in $G_{\mathrm{dir}}$ corresponds to a sequence of edges and vertices in $G$. Furthermore, a transition from  $(v_1,n_1)$ to $(v_2,n_2)$ exists if
  $v_1$ has neighbor $v_2$ with type $n_2$ in $G$ and $\phi$ evaluates to true.
  This is the case if $v_1$ is a collider with regard to $n_1, n_2$ and $v_1
  \in Z$ (left clause) or if $v_1$ is a non-collider with regard to $n_1,n_2$
  and $v_1 \in Z$ (right clause). Thus, directed paths in
  $G_{\mathrm{dir}}$ in fact correspond to d-connecting walks in $G$. For illustration, consider
  the DAG on the left of Figure \ref{figure:dsep:statespace} and suppose we
  have $L=\{\{v_1\},\{v_5\}\}$. The resulting state-space graph is shown on the
  right. We can see that $(v_4,\rightarrow)$ is reachable from $(v_1,\leftarrow)$
  in the state-space graph; see the path highlighted in red. It corresponds to the d-connecting walk $v_1 \rightarrow v_2 \rightarrow v_5 \leftarrow v_2 \leftarrow v_3 \rightarrow v_4$. Therefore $v_1$ and $v_4$ are d-connected given
  $\{v_5\}$ in the original DAG. In other words, we have translated
  d-separation into a \textsc{reach} task.
\end{example}

Applying reachability to state-space graphs like the one in
Figure~\ref{figure:dsep:statespace} yields a linear-time algorithm for finding
all d-connected nodes. Of course, this corresponds to the well-known Bayes-Ball
algorithm~\citep{shachter1998bayes} which can therefore be seen as a CIfly
algorithm. We now formally prove that, in fact, \emph{every} CIfly reduction
admits a linear-time algorithm.

\begin{algorithm}[t]
  \DontPrintSemicolon
  \SetKwInOut{Input}{input}\SetKwInOut{Output}{output}
  \SetKwInput{Specification}{specification}
  \SetKwFunction{FPass}{$\phi$}
  \Specification{$(\mathcal{E}, \ell)$-CIfly reduction given by $(C,
    \mathcal{S}, \mathcal{T}, \phi)$.}
  \Input{A graph $G = (V, (E_1, \dots, E_{\ell}))$ with neighbor-types
    $\mathcal{N}_{\mathcal{E}}$ and a list of sets $L = (L_1, \dots, L_{\ell})$
    with $L_i \subseteq V$.}
  \Output{Set of nodes $R$.}
  \SetKwFunction{FVisit}{visit}
  \SetKwProg{Fn}{function}{}{end}

  \BlankLine
  Initialize $\mathrm{visited}[(v, n, c)]$ with \texttt{false} for all $v \in
  V$ and $n \in \mathcal{N}_{\mathcal{E}}$ and $c \in C$. \;
  \BlankLine
  \Fn(\tcp*[h]{is called at most $p \cdot |\mathcal{N}_{\mathcal{E}}| \cdot |C|$ times})
  {\FVisit{$(v_1, n_1, c_1)$}}{
    $\mathrm{visited}[(v_1, n_1, c_1)] := \texttt{true}$ \;
    \ForEach{$v_2 \in N_{n_2}(v_1)$ in $G$}{
      \ForEach{$c_2 \in C$}{\label{line:ciflydfs:colors} \tcp{in total, $\phi$ is called at most $m \cdot |\mathcal{N}_{\mathcal{E}}| \cdot |C|^2$ times}
        \If{\textbf{\emph{not}} $\mathrm{visited}[(v_2, n_2, c_2)]$ \textbf{\emph{and}}
          $\FPass(v_1 \in L_1, \dots, v_1 \in L_{\ell}, v_2 \in L_1, \dots, v_2 \in L_{\ell}, n_1, n_2, c_1, c_2)$ \label{line:ciflydfs:phi}}{
          \FVisit{$(v_2, n_2, c_2)$}
        }
      }
    }
  }
  \BlankLine
  \ForEach{$(i, n, c) \in \mathcal{S}$}{
    \ForEach{$v \in L_i$}{
      \lIf{\textbf{\emph{not}} $\mathrm{visited}[(v, n, c)]$}{
        \FVisit{$(v, n, c)$}
      }
    }
  }
  \Return $\{v \mid \mathrm{visited}[(v, n, c)] = \texttt{true} \text{ and }
  (n, c) \in \mathcal{T}\}$\;
 \caption{Linear-time algorithm performing a CIfly reduction.}
  \label{algorithm:ciflydfs}
\end{algorithm}

\begin{theorem}\label{theorem:lintime}
  Let $(C, \mathcal{S}, \mathcal{T}, \phi)$ be a $(\mathcal{E}, \ell)$-CIfly
  reduction and $(G, (L_1, \dots, L_{\ell}))$ an input. Then,
  Algorithm~\ref{algorithm:ciflydfs} computes the output set of nodes $R$ in
  time $O(p+m)$, that is, in linear time in the size of $G$.  
\end{theorem}

\begin{proof}
  Algorithm~\ref{algorithm:ciflydfs} directly performs the CIfly reduction by
  calling the recursive \texttt{visit} function for each state given by set
  $\mathcal{S}$. This effectively performs a DFS version of the \textsc{reach}
  algorithm on $G_{\mathrm{dir}}$ with the only modification being that
  $G_{\mathrm{dir}}$ is not preconstructed but its edges are generated during
  the traversal. Lastly, the set $R$ is computed using $\mathcal{T}$.

  We now analyze the time complexity of the algorithm. First, observe that the
  state space has, by definition, size $p \cdot |\mathcal{N}_{\mathcal{E}}|
  \cdot |C|$ and, as we assume that the number of edge types is constant and
  define $C$ to have constant size as well, this is in $O(p)$. Hence, the
  \texttt{visit} function is called at most $O(p)$ times because this happens
  only if a triple $(v, n, c)$ has not been visited previously. Second, the
  number of possible transitions that are evaluated in
  line~\ref{line:ciflydfs:phi} by calling the function $\phi$ is upper bounded
  by $2 \cdot m \cdot |\mathcal{N}_{\mathcal{E}}| \cdot |C|^2$, because each
  edge is considered at most twice with degrees of freedom for $n_1$, $c_1$ and
  $c_2$ but not for $n_2$ as it is determined by the edge's type. This upper
  bound is in $O(m)$ by the same arguments as above. Third, the function $\phi$
  can be evaluated in constant time because (i) its inputs can be retrieved in
  constant time by look-ups and (ii) any Boolean function with constant input
  and output size can be computed in constant time as well.
\end{proof}

\subsection{Specifying CIfly Algorithms with Rule Tables}
\label{section:rule tables}
The specification of an $(\mathcal{E}, \ell)$-CIfly reduction, $(C,
\mathcal{S}, \mathcal{T}, \phi)$, contains all the information required to
implement the corresponding \textsc{reach} algorithm as defined in Algorithm
\ref{algorithm:ciflydfs}. However, such a specification is obviously a complex
object; particularly the function $\phi$ specifying the transitions.
\citet{wienobst2024linear} and \citet{henckel2024graphical} have shown that it
is possible to encode such a function with an object they call a \emph{rule
table}. We now adapt and formalize this idea to develop notation for succinctly
representing a CIfly specification.

Specifically, a CIfly rule table consists of three columns. The first two
contain pairs of neighbor-type and color. In Algorithm
\ref{algorithm:ciflydfs}, these correspond to the current state, specifically
its values $n_1$ and $c_1$, and the next state, specifically its values $n_2$
and $c_2$. The third column contains those parts of $\phi$ that depend only on
the membership of $v_1$ and $v_2$ in the sets in $L$, hence, it can be seen as
a function $\phi_{n_1, c_1, n_2, c_2}: \{0, 1\}^{2\ell} \mapsto \{0, 1\}$. To
avoid confusion between $v_1$ and $v_2$, we write \emph{current} and
\emph{next} instead of $v_1$ and $v_2$ when specifying a rule table. As a
simple example of a rule table row consider $\tikzleftarrow$ and \texttt{blue}
in the first column, $\tikzleftrightarrow$ and \texttt{red} in the second
column and $\text{next} \in L_1 \land \text{current} \not\in L_2$ in the third.
This encodes, that if $(v_1, \tikzleftarrow, \texttt{blue})$ is a state in the
state-space graph and $v_1$ has a neighbor $v_2$ connected via an edge of the
form $\tikzleftrightarrow$ in $G$, then there is a transition to state $(v_2,
\tikzleftrightarrow, \texttt{red})$ if $v_2 \in L_1 \land v_1 \not\in L_2$.  We
refer to a row of the rule table as a \emph{rule} with the first two columns
constituting the \emph{case}, which is matched against to decide which rule is
used, and the third column the \emph{expression}, which is evaluated to decide
whether there exists a transition. 

To streamline the notation and ease presentation, we introduce three additional
notational innovations. First, we allow the specification of multiple
neighbor-types and colors within a single rule, and the rule then matches any
combination of the given neighbor-types and colors. Hence, one could have
$\tikzleftarrow, \tikzleftrightarrow$ in the first column to match both
neighbor-types.  As a shorthand for the list of all neighbor-types or colors,
we write \emph{any}. Second, we allow multiple rules matching the same
combination of neighbor-types and colors. In such cases, the expression of the
\emph{first} matching rule is evaluated to decide whether the transition
exists. If there is no matching rule, we assume that there is no transition.
Furthermore, we ignore colors altogether, if they are not used in a
specification. Third and finally, we add the additional information required to
specify $\mathcal{E}$, $\ell$, $C$, $\mathcal{S}$ and $\mathcal{T}$. Instead of
just stating $\ell$, we usually give symbols to the sets in $L$ and later refer
to these. The mappings $\mathcal{S}$ and $\mathcal{T}$ are given informally in
the lines $\emph{Start}$ and $\emph{Return}$, referring to the introduced set
symbols instead of integers, using \emph{any} to refer to all neighbor-types and
omitting the color if not needed. If for $\mathcal{T}$ we return all reached states of a color irrespective of the neighbor-type we simply state the color. This is the case for all tables in this paper. By combining these elements, a rule table
fully specifies an $(\mathcal{E}, \ell)$-CIfly reduction $(C, \mathcal{S},
\mathcal{T},\phi)$. 

\begin{figure}
  \begin{minipage}{0.47\textwidth}
    \begin{tabular}{rrr}
      \multicolumn{3}{l}{\textbf{Sets}: $X, Z$}                                                                   \\
      \multicolumn{3}{l}{\textbf{Start}: $(X, \tikzleftarrow)$}                                                   \\
      \multicolumn{3}{l}{\textbf{Return}: any}                                                                    \\ [0.25em] \toprule
      current                                & next                                  & rule                       \\ \midrule
      $\tikzrightarrow, \tikzleftrightarrow$ & $\tikzleftarrow, \tikzleftrightarrow$ & $\text{current} \in Z$     \\
      any                                    & any                                   & $\text{current} \not\in Z$ \\ \bottomrule
    \end{tabular}
  \end{minipage}
  \begin{minipage}{0.52\textwidth}
    \vspace*{.48cm}
    \begin{lstlisting}[language=ruletable]
EDGES  --> <--, <->
SETS   X, Z
START  <-- AT X
OUTPUT ...

-->, <-> | <--, <-> | current in Z
...      | ...      | current not in Z
    \end{lstlisting}
  \end{minipage}
  \caption{CIfly rule table for computing all nodes d-connected with $X$ given
  $Z$ in an ADMG (left) and the text file version as parsed by the
  \texttt{cifly} package (right).} 
  \label{figure:ruletableformat}
\end{figure}

\begin{example}
\label{example:ADMG d-sep}
We illustrate the rule table format by using it to specify the CIfly primitive
for finding all nodes $y$ d-connected to $X$ given $Z$ in an ADMG on the left
of Figure~\ref{figure:ruletableformat}. Here, $\phi$ is succinctly described by
a mere two rows. One covers the case of encountering a collider and the other
the case of encountering a non-collider, that is whether or not the arrowheads
meet at the current node. Since we only use the first matching rule to evaluate
a transition, the second rule matches only non-colliders even though \emph{any}
is given as current and next. 

To make this more concrete, consider a potential transition from $(v_1, n_1)$
to $(v_2, n_2)$ and suppose $v_2 \in N_{n_2}(v_1)$. Here, we omit colors as
they are not used. To decide this transition, we look for the first row in
which $n_1$ matches in the column current and $n_2$ matches in the column next.
For example, for $n_1 = \tikzleftrightarrow$ and $n_2 = \tikzleftarrow$, this
would be the upper row (the bottom row also matches but it is not considered as
it is not the first match). Hence, we evaluate whether $v_1 \in Z$ (written in
the table as $\text{current} \in Z$) to decide whether a transition exists
between these two states. If instead $n_1 = \tikzleftrightarrow$ and $n_2 =
\tikzrightarrow$ then only the second row would match and we would evaluate
$v_1 \not\in Z$ to decide whether a transition exists.  
\end{example}

Even though CIfly reductions and algorithms are complex objects, the rule table
notation provides us with a concise way to summarize, communicate and compare
them. Deliberately, these table include syntactic sugar that make them
intuitive to use even without having the transition graph and the formal
reduction to \textsc{reach} in mind. For simple CIfly algorithms, such as for
finding all d-connected nodes (see Example \ref{example:bayesball}), this is
already helpful. For more complex graph types and problems, these
simplifications are crucial to keep the specification manageable and we will
rely on rule tables extensively for the remainder of this paper.

\subsection{The CIfly Software Packages}
\label{subsection:software}
The restrictions to the rule table specification of reachability algorithms in
CIfly are designed to allow a generic and convenient interface for running
those algorithms efficiently in practice. We provide such an implementation in
the \texttt{cifly} Rust package. Its functionalities are available in R and
Python through the \texttt{ciflyr} and \texttt{ciflypy} wrapper packages. The
main idea is that a CIfly algorithm, and the rule table at its center, can conveniently be
specified in a text-based format illustrated on the right of
Figure~\ref{figure:ruletableformat}. This representation is language-agnostic
and machine-readable. The corresponding reachability algorithm can then be
executed from R and from Python by calling a function called \texttt{reach}. 


This interface offers multiple advantages. First, because the core logic is
language-independent it can easily be used across project boundaries.
Currently, trying to port R code to Python, or vice versa, is time-consuming
and error-prone. Copying or reusing a rule-table file is simple. Second, our
implementation is designed with efficiency in mind. We transform the rules to a
format that allows a fast lookup of transitions, by precomputing, for each
triple $(n_1, c_1, n_2)$, the first matching row for each color $c_2$ that
matches at any row, and parse the expressions into abstract syntax trees before
running the reachability algorithm. The underlying rule table parser and
reachability algorithm is implemented in Rust, a compiled language
significantly faster than native R or Python code for running graph algorithms.
This methodology is similar to using efficient C and \Cpp{} libraries for
linear algebra primitives as is common in R and Python. In
Section~\ref{section:cifly:in:action}, we document the speed of our
implementation by comparing it to several state of the art causality suites.
Third, restricting the rule table ensures that the specified algorithm runs in
linear time (see Theorem~\ref{theorem:lintime}). These guarantees translate to
our CIfly implementation as the text-file rule tables prohibit the
specification of a super-linear-time algorithm. 

\begin{figure}[t]
  \begin{minipage}{0.49\textwidth}
    \begin{lstlisting}[language=Python, style=codeStyle]
import ciflypy as cifly
dsep_table_path = "path/to/ruletable"

def test_dsep(G, X, Y, Z):
    R = cifly.reach(G, {"X": X, "Z": Z}, dsep_table_path)
    return y not in R 
    \end{lstlisting}
  \end{minipage}
  \hfill
  \begin{minipage}{0.49\textwidth}
    \begin{lstlisting}[language=R, style=codeStyle]
library("ciflyr")
dsepTablePath <- "path/to/ruletable"

test_dsep <- function(G, X, Y, Z) {
    R <- reach(G, list("X" = X, "Z" = Z), dsepTablePath)
    return (!(y %in% R))}
    \end{lstlisting}
  \end{minipage}
  \caption{Python (left) and R code (right), respectively, for checking whether
  $X\perp_G  Y \mid W$ in an ADMG $G$ by using the CIfly rule table given on
  the right of Figure~\ref{figure:ruletableformat}.}
  \label{figure:admgcode}
\end{figure}

Rule table files in the CIfly software mirror the rule table format as introduced in the
previous section. They consist of two parts: the first one specifies
$\mathcal{E}$, $\ell$, $C$, $\mathcal{S}$ and $\mathcal{T}$  and the second
part describes $\phi$ as a table. In the first part, each line starts with a
keyword in capital letters followed by the corresponding specification.
Notably, instead of ``any'' we write ``\texttt{...}'' and combine the edges and
neighbor-types in one line, with edges separated by commas. For ordered edges, both corresponding 
neighbor-types must be provided, separated by whitespace. In the second
part, the columns are separated by '\texttt{|}' and the expression in the third
column may use the variables \texttt{current} and \texttt{next}, the set
symbols defined in the first part and the logical operators \texttt{in},
\texttt{not in}, \texttt{and}, \texttt{or} and \texttt{not} combined with
parentheses '\texttt{(}' and '\texttt{)}'. For further details regarding the rule
table syntax we refer readers to the documentation at \url{cifly.pages.dev}.

A rule table can be stored as a text file or be embedded directly into Python
or R code. In Figure~\ref{figure:admgcode}, we present code for checking
d-separation statements in ADMGs. It loads the rule table from
Figure~\ref{figure:ruletableformat} via its file path. The \texttt{ciflypy} and
\texttt{ciflyr} packages provide a single function named \texttt{reach}, that
has to be called with the graph, the sets $L$ and the rule table (or a path to
it) and returns the output nodes of the causal reduction. The graph and sets
are represented by Python dictionaries and R lists, respectively. Again, we
refer to the documentation for the details and all further options of the
\texttt{reach} function. 

\section{CIfly in Action}
\label{section:cifly:in:action}
In this section, we illustrate CIfly from a software-oriented perspective by
using it to develop algorithms for two established problems regarding covariate
adjustment. First, verifying whether an adjustment set in a CPDAG is valid
\citep{perkovic2018complete} and second, computing the parent adjustment
identification distance between two CPDAGs \citep{henckel2024graphical}, the
latter being closely related to the more commonly known structural intervention
distance \citep{peters2015structural}. We focus on established problems as it
allows us to benchmark CIfly algorithms against state-of-the art software
implementations. We consider CPDAGs to illustrate the versatility of CIfly. Our
aim is to showcase three core strengths of CIfly: methodological clarity,
computational efficiency and language agnosticism.

\subsection{Verifying Adjustment Set Validity}
\label{section:adjustment}
Covariate adjustment is a popular approach for inferring causal effects from
observational data. Deciding which covariates to adjust for is a difficult and
fundamental problem. In causal graphical models, we can characterize which sets
of covariates are suitable for inferring causal effects via adjustment and such
sets are called valid adjustment sets. The following result characterizes valid
adjustment sets in CPDAGs.

\begin{figure}[t]
  \begin{minipage}{0.54\textwidth}
    \begin{tabular}{rlrlr}
      \multicolumn{5}{@{}p{2in}@{}}{\textbf{Sets}:   $X$}                                \\
      \multicolumn{5}{@{}p{2in}@{}}{\textbf{Start}:  $(X, -, \text{init})$}              \\
      \multicolumn{5}{@{}p{2in}@{}}{\textbf{Return}: yield}                              \\ [0.25em] \toprule
      \multicolumn{2}{c}{current} & \multicolumn{2}{c}{next} &                           \\ \cmidrule(r){1-2}\cmidrule(l){3-4}
      n.-t.\              & color & n.-t.\              & color & rule                   \\ \midrule
      $-$                 & init  & $-$                 & yield & $\text{next} \notin X$ \\
      $-,\tikzrightarrow$ & yield & $-,\tikzrightarrow$ & yield & $\text{next} \notin X$ \\ \bottomrule
    \end{tabular}
  \end{minipage}
  \begin{minipage}{0.45\textwidth}
    \centering
    \vspace*{-1.145cm}
    \begin{tabular}{rrr}
      \multicolumn{3}{@{}p{2in}@{}}{\textbf{Sets}: $X,W$}                               \\
      \multicolumn{3}{@{}p{2in}@{}}{\textbf{Start}: (X, $\tikzleftarrow$)}              \\
      \multicolumn{3}{@{}p{2in}@{}}{\textbf{Return}: any}                               \\ [0.25em] \toprule
      \multicolumn{1}{c}{current} & \multicolumn{1}{c}{next} & \multicolumn{1}{r}{rule} \\ \midrule
      $\tikzleftarrow,-$          & $\tikzleftarrow,-$       & $\text{next} \notin W$   \\ \bottomrule
    \end{tabular}
  \end{minipage}
  \caption{CIfly rule table for computing all nodes violating Condition \ref{condition: amenability} of Theorem \ref{theorem:adjustment} with respect to $X$ (left) and $\possan_{G,W}(X)$ (right), respectively, in a CPDAG.
  }
  \label{figure:adjustment simple}
\end{figure}

\begin{restatable}[\citealt{perkovic2018complete}]{theorem}{CpdagAdjustmentCriterion}
  \label{theorem:adjustment}
  Let $X,Y$ and $W$ be pairwise disjoint node sets in a CPDAG $G$. Then, $W$ is a
  valid adjustment set relative to $(X,Y)$ in $G$ if and only if
  \begin{enumerate}
    \item all proper possibly directed paths from $X$ to $Y$ begin with a
      directed edge, \label{condition: amenability}
    \item $\forb_G(X, Y) \cap W = \emptyset$, and \label{condition: forbidden}
    \item $X \perp_{\tilde{G}} Y \mid W$, where $\tilde{G}$ is $G$ with all
      directed edges from nodes in $X$ to nodes in $causal_G(X,Y)$ removed.
    \label{condition: non-causal adjustment}
  \end{enumerate}
\end{restatable}

\begin{table}[t]
  \begin{center}
    \begin{tabular}{rrr}
      \multicolumn{3}{@{}p{2in}@{}}{\textbf{Sets}: $X,W,C$}                                                                                      \\
      \multicolumn{3}{@{}p{2in}@{}}{\textbf{Start}: (X, $\tikzleftarrow$)}                                                                       \\
      \multicolumn{3}{@{}p{2in}@{}}{\textbf{Return}: any}                                                                                        \\ [0.25em] \toprule
      \multicolumn{1}{c}{current} & \multicolumn{1}{c}{next} & \multicolumn{1}{r}{rule}                                                          \\ \midrule
      $\tikzrightarrow$           & $\tikzleftarrow$         & $\text{current} \in W$                                                            \\
      $-$                         & $-$                      & $\text{current} \notin W$                                                         \\
      any                         & $\tikzrightarrow$        & $\text{current} \notin W \text{ and (current} \notin X \text{ or next} \notin C)$ \\
      $\tikzleftarrow$            & any                      & $\text{current} \notin W$                                                         \\ \bottomrule
    \end{tabular}
  \end{center} \label{table:non-causal adjustment}
  \caption{CIfly rule table for finding all nodes violating Condition
  \ref{condition: non-causal adjustment} of Theorem \ref{theorem:adjustment}
  with respect to $X,W$ and $C\coloneq\causal_G(X,Y)$ in a CPDAG.}
\end{table}

We are interested in the algorithmic problem of deciding whether $W$ is a valid
adjustment set relative to $(X,Y)$ in $G$.  \cite{perkovic2018complete} provide
two implementations to solve this problem: one comparably slow implementation
called \texttt{gac} provided in \texttt{pcalg}~\citep{kalisch2012causal} and
one more efficient version called \texttt{isAdjustmentSet} provided in
\texttt{DAGitty}~\citep{textor2016robust}. We now use CIfly to develop an
algorithm for the same task. It proceeds by computing the sets $C_1$,
consisting of the nodes $y$ violating Condition \ref{condition: amenability}
relative to $X$, $C_2 = \forb_G(X, Y) $, and finally, $C_3$, consisting of the
nodes $y$ violating Condition \ref{condition: non-causal adjustment} relative
to $X$ and $W$. We then check whether $Y \cap (C_1 \cup C_3) = \emptyset$ and
$W \cap C_2 =\emptyset$. If both are true, then we return that $W$ is a valid
adjustment. 

To compute $C_1$, we need to find all $y$ such that there exists a proper
possibly directed path from some $x \in X$ to $y$ that starts with an
undirected edge. This is very similar to standard reachability but there are two additional
subtleties. First, the path's starting edge has to be undirected, whereas all
other edges may be either undirected or directed away from $x$. Second, we need
to track whether we encounter another node in $X$. In CIfly we can accommodate
the former, by introducing a separate color which we use for initializing the
reachable algorithm with a special starting rule. We can accommodate the latter
by adding a check whether the current node is in $X$ to the transition rules.
The resulting rule table is given on the left of Figure \ref{figure:adjustment
simple}. 

To compute $C_2$, we first find $\causal_G(X,Y)=\possan_{G,W}(X) \cap
\possde_G(Y)$. We can compute the unconventional $\possan_{G,W}(X)$ with a
simple CIfly algorithm that follows up possibly directed paths that do not
contain nodes in $W$. We provide the corresponding rule table on the right of
Figure \ref{figure:adjustment simple}. With a similar rule table, we can also
compute $\possde_G(Y)$. Once we have learned $\causal_G(X,Y)$, we can use the
latter table again to compute $C_2=\forb_G(X,Y)$. Overall, we therefore rely on
three simple CIfly calls to compute $\forb_G(X,Y)$. 

Finally, we compute $C_3$ by re-implementing the Bayes-Ball algorithm for
CPDAGs \citep{van2020algorithmics} in CIfly with one modification: we ignore
edges from $X$ to $\causal_G(X,Y)$. We provide the rule table for this in Table
\ref{table:non-causal adjustment}. The algorithm relies on a characterization
of d-separation in CPDAGs via walks that do not contain segments of the form
$\tikzrightarrow v \ -$ of $-\ v \tikzleftarrow$ called almost-definite-status
walks. An almost-definite-status walk is open given $W$ if all colliders are in
$W$ and none of the other nodes are. The rules in Table \ref{table:non-causal adjustment} 
check the former in Rule 1 and the latter in Rules 2, 3 and 4. For
more details regarding this characterization of d-separation in CPDAGs, 
see Section 3.1.2 of \citet{van2020algorithmics}. Notably, we avoid
constructing the graph $\tilde{G}$ by adapting Rule 3 to ignore edges from $X$
to $\causal_G(X,Y)$. This is not necessary but good practice as it is easy to
inadvertently implement a super-linear algorithm for computing $\tilde{G}$.

The algorithm uses five CIfly calls in total. In each, we exploit the
versatility of CIfly in terms of transition functions and, when necessary, use
colors to seamlessly solve algorithmic tasks that are similar to
reachability. The ubiquity of such problems makes
CIfly broadly applicable in graphical causal inference. We provide R code for
this algorithm in Figure~\ref{figure:gacalgorithm} and Python code using the
same rule tables under~\url{cifly.pages.dev}; the first Python implementation
for this task made trivial by CIfly.

\begin{figure}
  \begin{lstlisting}[language=R, style=codeStyle]
library(cifly)
# cpdag is the graph and X, Y, W are sets
# notAmenable, possibleAncestors, ... are rule tables
cpdagAdjustmentCheck <- function(cpdag, X, Y, W) {
  notAmenable <- reach(cpdag, list("X" = X), notAmenable)
  anc <- reach(cpdag, list("X" = Y, "W" = X), possibleAncestors)
  des <- reach(cpdag, list("X" = X), possibleDescendants)
  causal <- intersect(anc, des)
  forb <- reach(cpdag, list("X" = causal), possibleDescendants)}
  dsConn <- reach(cpdag, list("X" = X, "W" = W), backdoorConnected)
  return (length(intersect(notAmenable, Y)) == 0 && length(intersect(forb, W)) == 0 && length(intersect(dsConn, Y)) == 0)}
  \end{lstlisting}
  \caption{R code for verifying whether $W$ is a valid adjustment set relative
  to $(X,Y)$ in a CPDAG. The code relies on five separate CIfly \texttt{reach}
  calls chained with basic logical and set operators.}
  \label{figure:gacalgorithm}
\end{figure}

We compare the run-time performance of our new implementation to both
\texttt{gac} and \texttt{isAdjustmentSet} in a small simulation study. For the
experiments we generate CPDAGs with expected degree 4, choose two nodes $x$ and
$y$ randomly and draw a $W$ of size 5; exact details are provided in
Appendix \ref{appendix:adjustment simulation}. We then run the function
\texttt{gac} in \texttt{pcalg}, \texttt{isAdjustmentSet} in \texttt{DAGitty}
and \texttt{cpdagAdjustmentCheck} from Figure~\ref{figure:gacalgorithm} to
check whether $W$ is a valid adjustment relative to $(x,y)$. The resulting
run times given in Table \ref{table:runtimes} are averages over $20$
repetitions of this procedure. They show that the CIfly run time grows roughly
linearly in the instance size and is, for larger instances, up to five orders
of magnitudes faster than the \texttt{pcalg} implementation, whose run time
clearly grows super-linear, and a factor 15 faster than \texttt{DAGitty}. 

\subsection{Computing the Parent Adjustment Identification Distance}
\label{section:parent adjustment}
Verifying whether an adjustment set is valid is also a crucial ingredient in a
more complex algorithmic task: computing adjustment identification distances
between pairs of graphs \citep{henckel2024adjustment}. These are a class of
distances between causal graphs that can be used to benchmark the performance
of causal discovery algorithms. Informally, they count the number of mistakes
one would make if one used a candidate graph $G_{\text{guess}}$ to select valid
adjustment sets but the ground truth graph is actually $G_{\text{true}}$. Here,
we focus on the parent adjustment identification distance for CPDAGs, where the
selected adjustment sets are the parents of the treatment.
\citet{henckel2024adjustment} provide an $O(p \cdot (p + m))$ algorithm for
computing this distance. They implement this algorithm in their software
package \texttt{gadjid} with a focus on run-time performance. In contrast to
common practice in the causal inference literature, this algorithm is
implemented in a \emph{compiled} language, namely Rust, with a Python wrapper
being offered for accessibility to end users, similar to CIfly. We re-implement
this algorithm using CIfly with details provided in
Appendix~\ref{appendix:gadjid} and Python code for the overall algorithm
provided in Figure \ref{cifly:code:parent:aid}.

\begin{table}
  \centering
  \begin{tabular}{cccc} \toprule
    \multicolumn{4}{c}{CPDAG adjustment criterion}               \\ \midrule
          & \multicolumn{3}{c}{Singleton $x$, $y$ and $|W| = 5$} \\ \cmidrule{2-4}
    $p$   & \texttt{pcalg} & \texttt{DAGitty} & CIfly            \\ \midrule
    $100$ & $2.60$        & $0.0235$         & $0.0006$         \\
    $200$ & $13.61$       & $0.0369$         & $0.0009$         \\
    $300$ & $40.53$       & $0.0531$         & $0.0012$         \\
    $400$ & $87.84$       & $0.0647$         & $0.0013$         \\
    $500$ & $156.80$      & $0.0744$         & $0.0014$         \\ \bottomrule
  \end{tabular}\hspace{1cm}
  \begin{tabular}{crrrr} \toprule
    \multicolumn{5}{c}{Parent adjustment distance}                 \\ \midrule
          & \multicolumn{2}{c}{Sparse} & \multicolumn{2}{c}{Dense} \\ \cmidrule(r){2-3}\cmidrule(l){4-5}
    $p$   & \texttt{gadjid} & CIfly    & \texttt{gadjid} & CIfly   \\ \midrule
    $100$ & $0.002$         & $0.005$  & $0.007$         & $0.008$ \\
    $200$ & $0.005$         & $0.012$  & $0.067$         & $0.039$ \\
    $300$ & $0.009$         & $0.024$  & $0.226$         & $0.116$ \\
    $400$ & $0.013$         & $0.038$  & $0.586$         & $0.238$ \\
    $500$ & $0.020$         & $0.055$  & $1.140$         & $0.426$ \\ \bottomrule
  \end{tabular}
  \caption{A run-time comparison of 
    CIfly, \texttt{DAGitty} and \texttt{pcalg}
    for checking the CPDAG adjustment criterion (left), and of CIfly and
    \texttt{gadjid} for computing the parent adjustment identification
    distance (right). Run times are reported in seconds.
  }
  \label{table:runtimes}
\end{table}

We compare the single-thread run-time performance of our CIfly re-implementation to
\texttt{gadjid} in a small simulation study. For the experiments, we generate CPDAGs
with $d=4$ (sparse graphs) and $d = p / 10$ (dense graphs) for $p$ equal to
$100$, $200$, $300$, $400$ and $500$; exact details are provided in Appendix
\ref{appendix:adjustment simulation}. We then run the functions \texttt{gadjid}
and \texttt{parentAdjustmentDistance} from Figure \ref{cifly:code:parent:aid}
to compute the parent adjustment distance between two CPDAGs. The resulting
run times, given on the right of Table \ref{table:runtimes}, are averages over
20 repetitions of this procedure. They show that CIfly is 2-3 times slower than
\texttt{gadjid} on sparse instances but slightly faster for dense ones, where
the number of edges grows quadratically in $p$. This shows we can match the
performance of optimized implementations with CIfly, even though CIfly is an
all-purpose tool and not optimized for any specific task. Users can therefore
focus exclusively on the causal path-based logic of the tables while still
reaping the rewards of an optimized implementation in terms of performance.
Moreover, by relying on the \texttt{ciflypy} and \texttt{ciflyr} packages
language porting is trivialized. In fact, the \texttt{gadjid} package does not
offer an R wrapper and therefore our CIfly implementation provides the first R
code for computing the parent adjustment distance for CPDAGs. 

\section{Developing Algorithms for Conditional Instrumental Sets in CIfly}
\label{section:iv}
In this section, we use CIfly to develop novel linear-time algorithms for
problems regarding \emph{instrumental variables}. Instrumental variables are
auxiliary covariates that can be used to estimate causal effects in the
presence of unmeasured confounding \citep{bowden1990instrumental}. They are
popular in econometrics and the health sciences
\citep{angrist1996identification,hernan2006instruments}. Considerable effort
has been devoted to understanding under what exact graphical conditions a
two-tuple $(Z,W)$ of covariate sets is a valid conditional instrumental set, in
the sense that it can be used to identify a causal effect from observational
data \citep{brito2002generalized,pearl2009causality}. There also exists an
$O(p \cdot (p+m))$ algorithm \citep{van2015efficiently,van2020algorithmics} for
finding a tuple satisfying the criterion by \citet{pearl2009causality}.
Recently, \citet{henckel2024graphical} have developed two novel results. First,
a necessary and sufficient graphical validity criterion for conditional
instrumental sets. Second, a procedure to compute a specific tuple $(\zopt,
\wopt)$ that under mild conditions is guaranteed to be valid and has a strong
statistical efficiency guarantee they call graphical optimality. However, no
efficient algorithms currently exist to either verify the new complete validity
criterion or compute the graphically optimal tuple $(\zopt, \wopt)$. In the
following, we use CIfly to develop linear-time algorithms for both tasks and
close this gap in the algorithmic literature. We also adapt the algorithm by
\citet{van2015efficiently} for finding valid conditional instrumental sets to
the criterion by \citet{henckel2024graphical} to obtain the first sound and
complete algorithm for this task. For space reasons we defer the latter
algorithm to Appendix~\ref{appendix:iv:sound:and:complete:finder} and this
section's proofs to Appendix~\ref{appendix:proofIV}. In
Appendix~\ref{appendix:iv:simulations}, we compare these algorithms with the
algorithm implemented in the \texttt{DAGitty} software package in a simulation
study, focusing on the practical run time and how often a valid conditional
instrument is returned. 

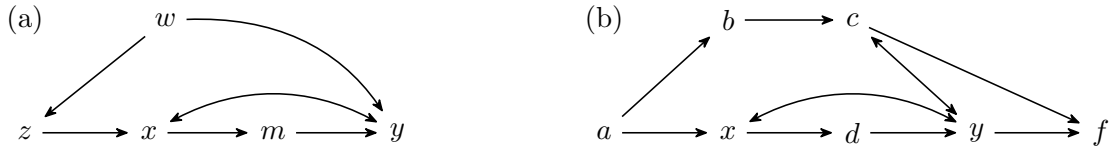
\begin{figure}[t]
  \centering
  \begin{tikzpicture}[xscale=1.1]
    \node (x)     at (0,0) {$x$};
    \node (m)     at ($(x)+(1.5,0)$) {$m$};
    \node (y)     at ($(x)+(3,0)$) {$y$};
    \node (z)     at ($(x)+(-1.5,0)$)   {$z$};
    \node (w)     at ($(x)+(0.2,1.5)$)   {$w$};
    \node (label) at ($(x) + (-1.5,1.5)$) {(a)};

    \path[arc]   (x) edge            (m);
    \path[arc]   (m) edge            (y);
    \path[arc]   (z) edge            (x);
    \path[arc]   (w) edge            (z);
    \path[arc]   (w) edge[bend left] (y);
    \path[bidir] (x) edge[bend left] (y);

    \node (x)     at (7,0)                {$x$};
    \node (m)     at ($(x) + (1.5,0)$)    {$d$};
    \node (y)     at ($(m) + (1.5,0)$)    {$y$};
    \node (d)     at ($(y) + (1.5,0)$)    {$f$};
    \node (w1)    at ($(x) + (0,1.5)$)    {$b$};
    \node (z)     at ($(x) + (-1.5,0)$)   {$a$};
    \node (w2)    at ($(x) + (1.5,1.5)$)  {$c$};
    \node (label) at ($(x) + (-1.5,1.5)$) {(b)};

    \path[arc]   (x)  edge            (m);
    \path[arc]   (m)  edge            (y);
    \path[arc]   (y)  edge            (d);
    \path[bidir] (x)  edge[bend left] (y);
    \path[arc]   (z)  edge            (x);
    \path[arc]   (z)  edge            (w1);
    \path[arc]   (w1) edge            (w2);
    \path[bidir] (w2) edge            (y);
    \path[arc]   (w2) edge            (d);
  \end{tikzpicture}
  \caption{Two ADMGs discussed in Example~\ref{example:validIV} and
  Example~\ref{example:optimalIV}, respectively.}
  \label{figure:graphsIV}
\end{figure}

\subsection{Conditional Instrumental Set Preliminaries}
\label{subsection:prelimsIV}
The graphical literature on conditional instrumental sets generally studies
them in the context of linear structural equation models with correlated errors
\citep{brito2002generalized,henckel2024graphical}. Such a model can be
represented by an ADMG, with directed edges representing direct causal effects
and bidirected edges latent variable induced dependencies between errors. For a
structural equation model, we can rigorously define the causal effect of a
treatment $x$ on an outcome $y$. Given the corresponding ADMG $G$, we can then
define valid conditional instrumental sets as those tuples $(Z,W)$ that allow
for consistent estimation of this causal effect for almost all models with ADMG $G$.
It it also possible to investigate which conditional instrumental sets result
in statistically efficient estimators for the target causal effect. We can
describe which conditional instrumental sets are valid or statistically
efficient using graphical criteria in $G$. We now introduce two such criteria. 
As only the graphical criteria are relevant for algorithmic development, we
refer readers interested in the precise definitions of these object to the
original publication by \citet{henckel2024adjustment} or, for a wider discussion
of causal models, to \citet{pearl2009causality}.

\begin{theorem}{\normalfont \textbf{\citep{henckel2024graphical}}}
  \label{theorem:validcis}
  Let $\{x,y\},Z$ and $W$ be pairwise disjoint node sets in an ADMG $G$. Then, $(Z,W)$
  is a valid conditional instrumental set relative to $(x,y)$ in $G$ if, and
  only if,
  \begin{enumerate*}[label=(\roman*)]
    \item $(Z \cup W) \cap \forb_G(x, y) = \emptyset$, \label{condition:validcis:forb}
    \item $x \not\perp_{G} Z \mid W$ and               \label{condition:validcis:ztox}
    \item $y \perp_{\tilde{G}} Z \mid W$,              \label{condition:validcis:znottoy}
  \end{enumerate*}
  where $\tilde{G}$ is $G$ with all directed edges from $x$ to a node in
  $\causal_G(x,y)$ removed. 
\end{theorem}

To provide the definition of $(\zopt,\wopt)$ and the conditions under which it
is valid and graphically optimal, let $\dis_{G, W}(u)$ denote the set of all
nodes connected to $u$ in $G$ via a path of bidirected edges with no nodes in
$W$ and $\displus_{G, W}(u) \coloneq (\dis_{G, W}(u) \cup \pa_G(\dis_{G,
W}(u))) \setminus W$. 

\begin{theorem}{\normalfont \textbf{\citep{henckel2024graphical}}}
  \label{theorem:optimalcis}
  Let $x$ and $y$ be nodes in an ADMG $G$ such that $\de_G(x) = \{x, y\}$. Let
  $\wopt \coloneq \displus_{G, \{x\}}(y) \setminus \{x, y\}$ and $\zopt
  \coloneq \displus_{G, \{y\}}(x) \setminus (W^{\mathrm{opt}} \cup \{x, y\})$.
  \begin{enumerate}
    \item \label{cond:optimalcis:sound} If $\zopt \neq \emptyset$, then
      $(\zopt, \wopt)$ is a valid conditional instrumental set relative to $(x,
      y)$ in $G$.
    \item \label{cond:optimalcis:opt} If $\zopt \cap (\pa_{G}(x) \cup
      \sib_{G}(x)) \neq \emptyset$, then $(\zopt, \wopt)$ is also graphically
      optimal.
  \end{enumerate}
\end{theorem}

Theorem \ref{theorem:optimalcis} requires that $\de_G(x) = \{x, y\}$. This may
seem like a strong assumption. It is, however, based on the result by
\citet{henckel2024graphical} that covariates in $\de_G(x)$ are only part of
valid conditional instrumentals set if the causal effect of interest can also
be estimated with covariate adjustment. Since, in these settings, adjustment is
more statistically efficient than an instrumental-variable-based approach,
\citet{henckel2024graphical} use this as motivation to remove such covariates
from the graph using a graphical operation know as a latent projection
\citep{richardson2003markov} and assume $\de_G(x) = \{x, y\}$ without loss of
generality. 

\subsection{Verifying Conditional Instrumental Set Validity}
\label{subsection:validIV}
We use the CIfly framework to develop a novel verification algorithm for
conditional instrumental sets. It finds \emph{all} $y$ such that $(Z, W)$ is a
valid conditional instrumental set relative to $(x, y)$ in linear time and is a
sound and complete algorithm, as it, in contrast to previous
work~\citep{van2015efficiently}, relies on the necessary and sufficient
criterion given in Theorem~\ref{theorem:validcis} and recently developed
by~\citet{henckel2024graphical}.

We can formulate a linear-time algorithm for checking the criterion in Theorem
\ref{theorem:validcis} similar to the algorithm from Section
\ref{section:adjustment} for valid adjustment: Compute
$\forb_G(x, y)$ to verify Condition~\ref{condition:validcis:forb}; afterwards
perform d-separation checks in the original graph for
Condition~\ref{condition:validcis:ztox} and in the modified graph for
Condition~\ref{condition:validcis:znottoy}. However, $\forb_G(x, y)$ and the
graph manipulation depend on $x$ and $y$. To go beyond this and develop a
linear-time algorithm that, given $Z$, $W$ and $x$, finds all $y$ such that
$(Z, W)$ is a valid conditional instrumental set relative to $(x, y)$
simultaneously, we need to avoid computing $\forb_G(x, y)$ and the graph
manipulation. To circumvent these steps, we reformulate the graphical criterion
of Theorem \ref{theorem:validcis} into a statement explicitly about the
existence or absence of certain walks. 

\begin{restatable}[]{theorem}{ValidCISwalks}
  \label{theorem:walksIV}
  Let $\{x,y\},Z$ and $W$ be pairwise disjoint node sets in an ADMG $G$. Then
  $(Z,W)$ is a valid conditional instrumental set relative to $(x,y)$
  in $G$ if and only if
  \begin{enumerate}[label=(\roman*)]
    \item there does not exist a directed path from $x$ to $y$ that
      contains a node in $W$,
      \label{condition:causalblockedIV}
    \item there exists a walk from $Z$ to $x$ that is open given $W$, and
      \label{condition:powerIV}
    \item there does not exist a walk from $Z$ to $y$ open given
      $W$ that does not end with a segment of the form $x
      \tikzrightarrow \dots
      \tikzrightarrow y$.
      \label{condition:noncausalIV}
  \end{enumerate}
\end{restatable}

\begin{table}[t]
	\begin{center}
		\begin{tabular}{rlrlr}
			\multicolumn{5}{@{}p{4in}@{}}{\textbf{Sets}: $\{x\},W$}                                               \\
			\multicolumn{5}{@{}p{4in}@{}}{\textbf{Start}: $(x, \tikzrightarrow, \text{causal-open})$}             \\
			\multicolumn{5}{@{}p{4in}@{}}{\textbf{Return}: causal-blocked}                                        \\ [0.25em] \toprule
			\multicolumn{2}{c}{current state}   & \multicolumn{2}{c}{next state}     &                            \\ \cmidrule(r){1-2}\cmidrule(l){3-4}
			neighbor-type     & color           & neighbor-type\    & color          & rule                       \\ \midrule
			$\tikzrightarrow$ & causal-open     & $\tikzrightarrow$ & causal-open    & $\text{current} \not\in W$ \\
			$\tikzrightarrow$ & causal-open     & $\tikzrightarrow$ & causal-blocked & $\text{current} \in W$     \\
			$\tikzrightarrow$ & causal-blocked  & $\tikzrightarrow$ & causal-blocked & true                       \\ \bottomrule
		\end{tabular}
	\end{center}
	\caption{CIfly rule table for finding all nodes that violate
		Condition~\ref{condition:causalblockedIV} of Theorem~\ref{theorem:walksIV}
		with respect to $x$ and $W$ in an ADMG.}
	\label{table:causalblocked}
\end{table}

As the Condition~\ref{condition:powerIV} is independent of $y$, we can
check it directly and return $\emptyset$ if it is violated using the CIfly
algorithm for checking d-separation statements in ADMGs (see Figure
\ref{figure:admgcode}). Hence, it remains to verify the new
Conditions~\ref{condition:causalblockedIV} and \ref{condition:noncausalIV} for
all $y$ in linear time. We do so separately using the rule tables given in
Table~\ref{table:causalblocked} and~\ref{table:noncausal}, respectively. For
some intuition on Table \ref{table:causalblocked}, the corresponding CIfly
algorithm starts from $x$ and Rule 1 ensures that it follows directed paths
pointing away from $x$ in the original graph similar to a conventional
reachable algorithm. Once we reach a node in $W$, the color changes to
\textit{causal-blocked} by Rule 2, and as a result all subsequent nodes are
marked as \textit{causal-blocked} by Rule 3. Finally, the algorithm returns all
nodes marked as causal blocked, which are precisely those for which Condition
\ref{condition:causalblockedIV} is violated. For some intuition on Table
\ref{table:noncausal}, the corresponding CIfly algorithm starts from $Z$ and
has effectively the same rules as the CIfly algorithm for finding all
d-connected nodes in ADMGs (see Example \ref{example:ADMG d-sep}) with one
important distinction. Rules 4 and 5 ensure that after passing through $x$ and
following along edges of the form $\tikzrightarrow$, nodes are marked with
color \textit{causal-end}. Since the algorithm only returns those nodes reached
in color \textit{non-causal}, this ensures that nodes d-connected with $Z$
given $W$ due to paths ending with segments of the form $x \tikzrightarrow
\dots \tikzrightarrow y$ are \textit{not} returned. This ensures that the nodes
returned are precisely those for which Condition \ref{condition:noncausalIV} is
violated. 

We can thus compute the set of all $y$ such that $(Z,W)$ is a valid
conditional instrumental set relative to $(x,y)$ in $G$ as follows: Check
Condition~\ref{condition:powerIV} and return $\emptyset$ if it is violated. If
not, compute the sets $O_{\ref{condition:causalblockedIV}}$ and
$O_{\ref{condition:noncausalIV}}$, with the specifications from
Table~\ref{table:causalblocked} and~\ref{table:noncausal}, respectively, and
return $V \setminus \left(O_{\ref{condition:causalblockedIV}} \cup
  O_{\ref{condition:noncausalIV}} \cup W \cup Z \cup \{x\}\}\right)$.
  Therefore, the following holds.

\begin{table}[t]
  \begin{center}
    \begin{tabular}{rlrlr}
      \multicolumn{5}{@{}p{4in}@{}}{\textbf{Sets}:   $Z,\{x\},W$}                                                                                                  \\
      \multicolumn{5}{@{}p{4in}@{}}{\textbf{Start}:  $(Z, \tikzleftarrow, \text{non-causal})$}                                                                     \\
      \multicolumn{5}{@{}p{4in}@{}}{\textbf{Return}: non-causal }                                                                                                  \\ [0.25em] \toprule
      \multicolumn{2}{c}{current state}                          & \multicolumn{2}{c}{next state}                           &                                      \\ \cmidrule(r){1-2}\cmidrule(l){3-4}
      neighbor-type                            & color           & neighbor-type                           & color          & rule                                 \\ \midrule
      $\tikzrightarrow$, $\tikzleftrightarrow$ & any             & $\tikzleftarrow$, $\tikzleftrightarrow$ & non-causal     & $\text{current} \in W$               \\
      $\tikzleftarrow$                         & any             & $\tikzleftarrow$, $\tikzleftrightarrow$ & non-causal     & $\text{current} \notin W$            \\
      any                                      & non-causal      & $\tikzrightarrow$                       & non-causal     & $\text{current} \notin W \cup \{x\}$ \\
      any                                      & non-causal      & $\tikzrightarrow$                       & causal-end     & $\text{current} \in \{x\}$           \\
      $\tikzrightarrow$                        & causal-end      & $\tikzrightarrow$                       & causal-end     & $\text{current} \not\in W$           \\ \bottomrule
    \end{tabular}
  \end{center}
  \caption{CIfly rule table for finding all nodes violating
  Condition~\ref{condition:noncausalIV} of Theorem~\ref{theorem:walksIV} with
  respect to $x,Z$ and $W$ in an ADMG.}
  \label{table:noncausal}
\end{table}

\begin{restatable}[]{theorem}{ValidCISalgotheorem}
Let $Z$, $\{x\}$ and $W$ be pairwise disjoint node sets in an ADMG $G$. Then, there
exists an algorithm finding all $y$  such that $(Z, W)$ is a valid conditional
instrumental set relative to $(x, y)$ in time $O(p+m)$, that is, in linear time
in the size of $G$.
\end{restatable}

\begin{example}
\label{example:validIV}
Consider the graph $G$ from Figure~\ref{figure:graphsIV}(a). We aim to decide
whether the conditional instrumental set $(Z=\{z\}, W=\{w\})$ is valid relative
to $(x,m)$ or $(x,y)$ in $G$. First, we check whether $x \not\perp_{G} Z \mid
W$, that is, whether there exists a walk from $Z$ to $x$ open given
$W$ (see Figure \ref{figure:admgcode}). Due to the edge $z \tikzrightarrow x$ this is true.
Hence, we also check the other two conditions (else the algorithm returns $\emptyset$). To check Condition~\ref{condition:causalblockedIV},
we find all nodes $v$ for which there exists a directed path from $x$ to $v$
containing a node in $W$ using Table~\ref{table:causalblocked}. Here, this set is empty because $w$ is not on any
directed path starting from $x$. To check
Condition~\ref{condition:noncausalIV}, we find all nodes $v$ for which there
exists an open walk from $Z$ to $v$ \emph{not} ending with a directed path from
$x$ using Table~\ref{table:noncausal}. Again, this set is empty because for both $m$ and $y$ there exists an open
walk from $z$ but it ends with a directed path from $x$. It follows that $({z},
{w})$ is a valid conditional instrumental set relative to $(x,m)$ and $(x,y)$,
respectively.
\end{example}

\subsection{Finding the Graphically Optimal Conditional Instrumental Set}
\label{subsection:efficientIV}
We can also use the CIfly framework to develop a linear-time algorithm for
computing the graphically optimal conditional instrumental set $(\zopt,\wopt)$
proposed by \citet{henckel2024graphical}. The original statement (see
Theorem~\ref{theorem:optimalcis}) assumes that the underlying causal ADMG has
been transformed by removing all descendants of the treatment, except $y$,
using a latent projection~\citep{richardson2003markov}. As discussed in Section
\ref{subsection:prelimsIV}, this is reasonable from a methodological
perspective. However, from an algorithmic perspective it is problematic.
Specifically, we show in Section \ref{section:comparison} that computing a
latent projection is equivalent to Boolean matrix multiplication. Thus,
linear-time algorithms may not rely on latent projections. We now provide an
alternative characterization of $(\zopt,\wopt)$ that avoids latent projections.

\begin{figure}[t]
  \begin{minipage}{0.54\textwidth}
    \begin{tabular}{rlrlr}
      \multicolumn{5}{@{}p{2in}@{}}{\textbf{Sets}:   $\{s\}, A, B$}                                         \\
      \multicolumn{5}{@{}p{2in}@{}}{\textbf{Start}:  $\{(s, \tikzleftarrow, \text{pass})\}$}                \\
      \multicolumn{5}{@{}p{2in}@{}}{\textbf{Return}: yield}                                                 \\ [0.25em] \toprule
      \multicolumn{2}{c}{current}   & \multicolumn{2}{c}{next}                        &                     \\ \cmidrule(r){1-2}\cmidrule(l){3-4}
      n.-t.\                & color & n.-t.\                                  & color & rule                \\ \midrule
      $\tikzleftarrow$      & pass  & $\tikzleftarrow$                        & pass  & $\text{next} \in B$ \\
      $\tikzleftarrow$      & pass  & $\tikzleftrightarrow$, $\tikzleftarrow$ & yield & $\text{next} \in A$ \\
      $\tikzleftrightarrow$ & yield & $\tikzleftrightarrow$, $\tikzleftarrow$ & yield & $\text{next} \in A$ \\ \bottomrule
    \end{tabular}
  \end{minipage}
  \begin{minipage}{0.45\textwidth}
    \vspace*{0.35cm}
    \begin{lstlisting}[language=ruletable]
EDGES  --> <--, <->
# p, y are short for pass, yield
COLORS p, y
SETS   S, A, B
START  <-- [p] AT S
OUTPUT ... [y]
  
<-- [p] | <-- [p]      | next in B
<-- [p] | <->, <-- [y] | next in A
<-> [y] | <->, <-- [y] | next in A
    \end{lstlisting}
  \end{minipage}
  \caption{Rule table for (i) computing $\woptnew$ by setting $s \coloneq y$,
    $A$ as defined in Theorem~\ref{theorem:optimalCISnew}, and $B \coloneq D
    \setminus \{x\}$ and (ii) computing $\zoptnew$ with $s \coloneq x$, $A$,
    and $B \coloneq \emptyset$ in an ADMG. On the right, we show the corresponding text file to illustrate how to use colors in
    combination with the CIfly software. Specifically, it includes a line
    \texttt{COLORS} listing all colors and afterwards, we refer to colors in
    square brackets wherever they are required in the rules. Here, \texttt{p}
    stands for \texttt{pass} and \texttt{y} for \texttt{yield}.}
  \label{figure:optimalIVtables}
\end{figure}

\begin{restatable}[]{theorem}{OptimalCISnewtheorem}
\label{theorem:optimalCISnew}
Let $x$ and $y$ be nodes in an ADMG $G$ such that $y \in D \coloneq \de_G(x)$
and let $A \coloneq V \setminus D$. Moreover, define
\[
  \woptnew \coloneq \{v \in V \mid v \text{ is on a path } p
  \tikzrightarrow c_1 \tikzleftrightarrow \cdots \tikzleftrightarrow c_k
  \tikzleftrightarrow d_1 \tikzrightarrow \cdots \tikzrightarrow d_l
  \tikzrightarrow y \} \cap A
\]
where $p, c_1, \dots, c_k \in A$ and $d_1, \dots, d_l \in (D \setminus \{ x
\})$ possibly with $k=0$ or $l=0$ and
\[
  \zoptnew \coloneq \{v \in V \mid v \text{ is on a path } p
  \tikzrightarrow c_1 \tikzleftrightarrow \cdots \tikzleftrightarrow c_k
  \tikzleftrightarrow x\} \setminus \woptnew
\]
where $p, c_1, \dots, c_k \in A$ possibly with $k=0$. Then, $\woptnew$ and
$\zoptnew$ for graph $G$ are identical to $\wopt$ and $\zopt$ for graph
$\tilde{G}$ obtained from $G$ by a latent projection over $D \setminus \{x,
y\}$.
\end{restatable}

Here, we define $\zoptnew$ and $\woptnew$ as sets of nodes lying on certain
paths. This makes the problem of finding them naturally amenable to the CIfly
framework. To compute $\woptnew$ we start a CIfly algorithm from $y$ and go up
paths of the form given in the definition of $\woptnew$, while tracking whether
we are currently on a segment of the form $d_i \tikzrightarrow \dots
\tikzrightarrow y$ or have already entered the segment $p \tikzrightarrow c_1
\tikzleftrightarrow \dots \tikzleftrightarrow c_k$, using colors. We can
distinguish between the two cases if we have knowledge of $D$, which we can
trivially precompute with another CIfly call using conventional reachability
rules. By setting the algorithm to only return nodes encountered on the latter
segment, we obtain $\woptnew$. The rule table for this CIfly algorithm and its
text file version are shown in Figure~\ref{figure:optimalIVtables}. To obtain
$\woptnew$, we call this CIfly algorithm with inputs $s \coloneq y$, the set
$A$ as defined above and $B \coloneq D \setminus \{x\}$. 

We can also use the rule table in Figure \ref{figure:optimalIVtables} to
compute $\zoptnew$. To do so, we first run the CIfly algorithm with inputs $s
\coloneq x$, $A$ as defined above and $B \coloneq \emptyset$. This ensures that
the algorithm starts from $x$, follows up paths of the form $p \tikzrightarrow
c_1 \tikzleftrightarrow \dots \tikzleftrightarrow c_k$ and returns all reached
nodes. We then perform a post-processing step outside of CIfly to remove all
nodes in $\woptnew$ from the CIfly output and obtain $\zoptnew$. We provide
Python code implementing the full algorithm, including the two CIfly calls and
the post-processing step in Figure~\ref{figure:optimalIVcode}.

\begin{figure}
  \begin{lstlisting}[language=Python, style=codeStyle]
import ciflypy as cf
# p is the number of nodes, g is the ADMG, x and y are nodes
# descendants and optimal_iv are rule tables
def optimal_instrument(p, g, x, y):
  D = cf.reach(g, {"X": x}, descendants)
  A = set(range(p)) - set(de_x)
  Wo = set(cf.reach(g, {"S": y, "A": A, "B": D }, optimal_iv))
  Zo = set(cf.reach(g, {"S": x, "A": A, "B": []}, optimal_iv)) - Wo
  if y in D and Zo:
    return (list(Zo), list(Wo))
  else:
    return None
  \end{lstlisting}
  \caption{Python code for computing $(\zoptnew,\woptnew)$ in an ADMG as defined in
  Theorem~\ref{theorem:optimalCISnew}. We only verify validity
  (Condition~\ref{cond:optimalcis:sound}) and not graphical optimality
  (Condition~\ref{cond:optimalcis:opt}). }
  \label{figure:optimalIVcode}
\end{figure}

\begin{restatable}[]{theorem}{OptimalCISalgotheorem}
\label{theorem:optimalIValgo}
Let $x$ and $y$ be nodes in an ADMG $G$ such that $y \in \de_G(x)$. Then, there
exists an algorithm computing $(\zoptnew, \woptnew)$ in time $O(p + m)$, that
is, in linear time in the size of $G$.
\end{restatable}

\begin{example}
\label{example:optimalIV}
Consider the ADMG $G$ from Figure~\ref{figure:graphsIV}(b). We are interested in computing the graphically optimal
conditional instrumental set $(\zoptnew, \woptnew)$ relative to $(x,y)$ in $G$. Here, $D = \de_G(x)
= \{x,d,y,f\}$ can be computed by a trivial reachability call and therefore $A
= \{a, b, c\}$. Next, we compute $\woptnew$ by applying the rule table in
Figure~\ref{figure:optimalIVtables} with $s \coloneq y$, $A \coloneq \{a, b,
c\}$ and $B \coloneq \{d,y,f\}$. The algorithm traverses the paths $b
\tikzrightarrow c \tikzleftrightarrow y$ and $m \tikzrightarrow y$. Since $m
\in B$ it only marks $b$ and $c$ with color \textit{yield} and therefore
returns $\woptnew = \{b, c\}$. Afterwards, $\zoptnew$ is computed by setting $s
\coloneq x$, $A \coloneq \{a, b, c\}$ as before and $B \coloneq \emptyset$. The
path $a \tikzrightarrow x$ is traversed and $\{a\}$ is returned. Because $a
\not\in \woptnew$, we obtain $\zoptnew = \{a\}$. As $\zoptnew$ is non-empty and
contains a parent of $x$, it follows by Theorem \ref{theorem:optimalcis} that
$(\{a\},\{b,c\})$ is a valid and graphically optimal conditional instrumental
set relative to $(x,y)$ in $G$.
\end{example}

\section{The Complexity of Moralization and Latent Projection}
\label{section:comparison}
In this section, we discuss two other popular algorithmic primitives within
graphical causal inference, \emph{graph
moralization}~\citep{cowell2007probabilistic} and \emph{latent
projection}~\citep{richardson2003markov}.  Concretely, we prove that they are
computationally equivalent to Boolean matrix multiplication (\textsc{bmm}).
Currently, there is no algorithm for \textsc{bmm} that is faster than general
matrix multiplication algorithms with the best known run time being above
$O(p^{2.371})$ for square matrices of dimension $p$~\citep{alman2024more}. This
disproves contrary statements in the literature~\citep[for
example,][]{geiger1989dseparation,tian1998finding,van2014constructing} that claim
moralization can be solved in time $O(p^2)$. Hence, as opposed to
\textsc{reach}, which CIfly builds on, moralization and latent projection cannot
be used as primitives for linear-time algorithms. On the contrary, there exists
a practically reasonable implementation for latent projection that uses $2p$
CIfly calls and therefore runs in time $O(p \cdot (p + m))$.

\subsection{Moralization}
\emph{Graph moralization} is a classic technique in causal and probabilistic
inference~\citep{lauritzen1988local,verma1993deciding}. It transforms a DAG $D$
into an undirected graph $G$ while preserving certain conditional independence
statements encoded by $D$ via d-separation. Specifically, we can test whether
$X$ and $Y$ are d-separated given $Z$ in $D$ by checking if $Z$ is a node cut
between $X$ and $Y$ in the moralization of a subgraph of $D$. This reduction of
a complex graphical property, d-separation, to a more standard notion, cuts,
is, at first glance, quite similar to CIfly, which relies on reductions to
\textsc{reach} instances. The main differences are that moralization constructs
an undirected graph and that this new graph has the same node set as the
original graph. This makes moralization generally less flexible than CIfly that
constructs a \emph{directed} state-space graph which contains multiple states
per node in the original graph. Nonetheless, moralization can be used to solve
certain tasks that are difficult to solve with CIfly, one example being the
task of finding a d-separator between $X$ and $Y$ in a DAG $G$ of \emph{minimum
size}~\citep{tian1998finding}. In the following, we show that moralization is
computationally harder than stated in the causal inference literature. We begin
by formally introducing the notion of a \emph{moralized graph}.

\begin{definition}[Moralization]
Consider a DAG $D = (V, E)$. The \emph{moralized} graph $G$ of $D$, also known
as the moralization of $D$, is the undirected graph with node set $V$ and edge
set
$
  E' = \{ a - b \mid a \neq b \text{ and } (a \tikzrightarrow b \text{ or } a
  \tikzleftarrow b \text{ or } \exists c \in V \text{ s.t. } a \tikzrightarrow
c \tikzleftarrow b \text{ in } D) \}.
$
\end{definition}

\begin{example}
  Consider the bottom half of Figure~\ref{figure:bmmtomoralize}. Let $D$ be the DAG given on the left and $G$ the undirected graph on the right. Here, $G$ is the moralization of $D$. We obtain $G$ from $D$ by, first, adding additional undirected edges between pairs of nodes that have a child in common in $D$, such as between $a_1$ and $c_1$, and, second, making all directed edges undirected. 
\end{example}

We will refer to the computational problem of constructing the moralized graph
for a given directed graph $G=(V,E)$ as \textsc{moralize}. The naive algorithm
for \textsc{moralize} needs time $O(p^3)$: for each node, loop over all pairs
of parents and insert an edge between them, if not already present. Finally,
make all edges undirected. From a theoretical point-of-view, we can improve the
run time by using that if $A$ is the adjacency matrix of $G$, that is, the
matrix with $a_{ij}=1$ if and only if $i \tikzrightarrow j \in E$, then the
inserted edges correspond to the non-zeros of the matrix $AA^T$. However, this
requires advanced techniques, namely algebraic methods for matrix
multiplication based on~\cite{strassen1969gaussian} and even those are firmly
away from an $O(p^2)$ run time~\citep{alman2024more}. Such improvements to
achieve subcubic run time were not discussed in the causal inference
literature, where it was incorrectly presumed that a naive algorithm achieves
time $O(p^2)$~\citep[for
example,][]{geiger1989dseparation,tian1998finding,van2014constructing,jeong2022frontdoor}.
Here, we focus on the opposite direction, namely showing that Boolean matrix
multiplication (\textsc{bmm}) can be solved with graph moralization. This
formally certifies that \textsc{moralize} is computationally as hard as
\textsc{bmm}. The latter problem is deeply studied~\citep{fischer1971boolean,
yu2018improved} with no algorithm faster than general matrix multiplication
known. It is defined as follows:

\begin{definition}[Boolean matrix multiplication]
Given matrices $X \in \{0,1\}^{p_1 \times p_2}$ and $Y \in \{0,1\}^{p_2 \times
p_3}$ compute the Boolean matrix product $X \cdot Y \in\{0,1\}^{p_1\times p_3}$
whose entries for $i=1, \dots, p_1$ and $j=1, \dots, p_3$ are given by
\[
  (X\cdot Y)_{ij} = \bigvee_{k = 1}^{p_2} (X_{ik} \land Y_{kj}).
\]
\end{definition}

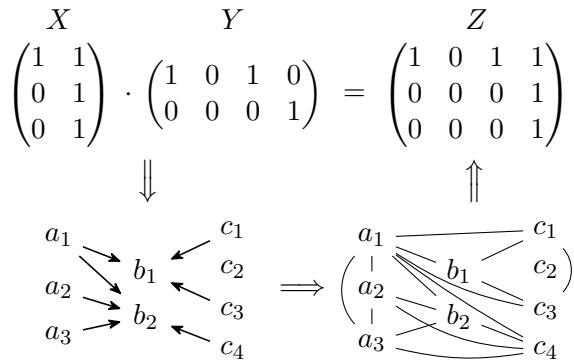
\begin{figure}[t]
\centering
\begin{tikzpicture}[xscale=0.925]
  \tikzset{circ/.append style={circle, fill=black, inner sep = 0,
  minimum size = 0.15cm}}
  \node (A) at (0,0) {$
    \begin{pmatrix} 1 & 1 \\ 0 & 1 \\ 0
      & 1
  \end{pmatrix}$};
  \node (B) at (2.5,0) {$
    \begin{pmatrix} 1 & 0 & 1 & 0 \\ 0 & 0
      & 0 & 1
  \end{pmatrix}$};
  \node (C) at (6,0) {$
    \begin{pmatrix} 1 & 0 & 1 & 1 \\
      0 & 0 & 0 & 1 \\
      0 & 0 & 0 & 1
  \end{pmatrix}$};
  \node (cdot) at (1,0) {$\cdot$};
  \node (eq) at (4.25,0) {$=$};
  \node (lp) at (0,1) {$X$};
  \node (lq) at (2.5,1) {$Y$};
  \node (lr) at (6,1) {$Z$};
  \node[rotate=270] at (1.25,-1.15) {$\Longrightarrow$};
  \node at (3.5,-2.6) {$\Longrightarrow$};
  \node[rotate=90] at (6,-1.15) {$\Longrightarrow$};
  \node (a1) at (0,-1.9) {$a_1$};
  \node (a2) at (0,-2.6) {$a_2$};
  \node (a3) at (0,-3.2) {$a_3$};
  \node (b1) at (1.25,-2.35) {$b_1$};
  \node (b2) at (1.25,-2.95) {$b_2$};
  \node (c1) at (2.5,-1.8) {$c_1$};
  \node (c2) at (2.5,-2.325) {$c_2$};
  \node (c3) at (2.5,-2.875) {$c_3$};
  \node (c4) at (2.5,-3.4) {$c_4$};
  \graph[use existing nodes, edge=arc] {
    a1 -> b1;
    a1 -> b2;
    a2 -> b2;
    a3 -> b2;
    c1 -> b1;
    c3 -> b1;
    c4 -> b2;
  };
  \node (a1) at (0+4.5,-1.9) {$a_1$};
  \node (a2) at (0+4.5,-2.6) {$a_2$};
  \node (a3) at (0+4.5,-3.3) {$a_3$};
  \node (b1) at (1.25+4.5,-2.35) {$b_1$};
  \node (b2) at (1.25+4.5,-2.95) {$b_2$};
  \node (c1) at (2.5+4.5,-1.8) {$c_1$};
  \node (c2) at (2.5+4.5,-2.325) {$c_2$};
  \node (c3) at (2.5+4.5,-2.875) {$c_3$};
  \node (c4) at (2.5+4.5,-3.4) {$c_4$};
  \graph[use existing nodes] {
    a1 -- b1;
    a1 -- b2;
    a2 -- b2;
    a3 -- b2;
    c1 -- b1;
    c3 -- b1;
    c4 -- b2;
    a1 -- a2;
    a1 --[bend right=45] a3;
    a2 -- a3;
    c1 --[bend left=45] c3;
    a1 -- c1;
    a1 --[bend right=10] c3;
    a1 --[bend right=3] c4;
    a2 --[bend right=14] c4;
    a3 --[bend right=10] c4;
  };
\end{tikzpicture}
\caption{Example of the Turing reduction from \textsc{bmm} to
\textsc{moralize}. The entry $Z_{ij}$ is $1$ if and only if there is an edge
$a_i - c_j$ in the moralized graph (bottom right).}
\label{figure:bmmtomoralize}
\end{figure}

We prove the following result with a straightforward reduction
illustrated in Figure~\ref{figure:bmmtomoralize}.

\begin{theorem}
Let $X \in \{0,1\}^{p_1 \times p_2}, Y \in \{0,1\}^{p_2 \times p_3}$ be two
Boolean matrices and assume there exists an $O(T(p))$ worst-case time algorithm
for moralizing a $p$ node graph, with $T: \mathbb{N} \mapsto \mathbb{N}$. Then,
the Boolean matrix product $Z = X \cdot Y$ can be computed in worst-case time
$O(T(p_1 + p_2 + p_3))$.
\end{theorem}

\begin{proof}
Construct a graph $G$ with node set $A \cup B \cup C$, where
$A=\{a_1,...,a_{p_1}\}$, $B=\{b_1,...,b_{p_2}\}$, and $C=\{c_1,...,c_{p_3}\}$.
Insert edge $a_{i} \tikzrightarrow b_{j}$ if $X_{ij} = 1$ and edge $b_{i}
\tikzleftarrow c_{j}$ if $Y_{ij} = 1$. Note that this graph is by definition
acyclic. Then, moralization will add an edge $a_{i} - c_{j}$ if, and only if,
$Z_{ij} = 1$ as this only happens in case there is at least one $k$ such that
$X_{ik} = 1$ and $Y_{kj} = 1$. For the run time, consider the three steps of
the sketched algorithm: (i) building graph $G$, (ii) performing moralization of
graph $G$ and (iii) constructing the resulting matrix $Z$. As $G$ has $p = p_1
+ p_2 + p_3$ nodes, (ii) takes time $O(T(p_1+p_2+p_3))$ by assumption. The time
complexity of (i) can be bounded by $O(p_1 \cdot p_2 + p_2 \cdot p_3)$, which
in turn is in $O((p_1 + p_2 + p_3)^2)$. Because $T(p)$ is necessarily in
$\Omega(p^2)$ (every algorithm for moralization has to add $\Omega(p^2)$ edges
in the worst-case), it follows that $O((p_1 + p_2 + p_3)^2)$ is in
$O(T(p_1+p_2+p_3))$. For (iii), similar arguments apply as it can be bounded by
$O(p_1 \cdot p_3)$.
\end{proof}

We conclude that moralization has a higher time complexity than previously
reported and is not suited as a primitive for linear-time algorithms in
contrast to CIfly primitives. 

\begin{figure}
  \centering
  \begin{tikzpicture}[>={Stealth[round,sep]}, xscale=1.1]
    \node (v1) at (6,3) {$v_1$};
    \node (v2) at (7,3) {$v_2$};
    \node (v3) at (8,3) {$v_3$};
    \node (v4) at (9,3) {$v_4$};
    
    \graph[use existing nodes, edges = {arc}] {
      v1 -- v2;
      v2 -- v3;
      v4 -- v3;
    };
    
    \node (v1s) at (6,1) {$v_1^s$};
    \node (v2s) at (7,1) {$v_2^s$};
    \node (v3s) at (8,1) {$v_3^s$};
    \node (v4s) at (9,1) {$v_4^s$};
    \node (v1) at (6,0) {$v_1$};
    \node (v2) at (7,0) {$v_2$};
    \node (v3) at (8,0) {$v_3$};
    \node (v4) at (9,0) {$v_4$};
    \node (v1t) at (6,-1) {$v_1^t$};
    \node (v2t) at (7,-1) {$v_2^t$};
    \node (v3t) at (8,-1) {$v_3^t$};
    \node (v4t) at (9,-1) {$v_4^t$};
    
    \graph[use existing nodes, edges = {arc}] {
      v1s -- v1;
      v2s -- v2;
      v3s -- v3;
      v4s -- v4;
      v1 -- v1t;
      v2 -- v2t;
      v3 -- v3t;
      v4 -- v4t;
      v1 -- v2;
      v2 -- v3;
      v4 -- v3;
    };
    
    \node (v1s) at (12,1) {$v_1^s$};
    \node (v2s) at (13,1) {$v_2^s$};
    \node (v3s) at (14,1) {$v_3^s$};
    \node (v4s) at (15,1) {$v_4^s$};
    \node (v1t) at (12,-1) {$v_1^t$};
    \node (v2t) at (13,-1) {$v_2^t$};
    \node (v3t) at (14,-1) {$v_3^t$};
    \node (v4t) at (15,-1) {$v_4^t$};
    
    \graph[use existing nodes, edges = {arc}] {
      v1s -- v1t;
      v2s -- v2t;
      v3s -- v3t;
      v4s -- v4t;
      v1s -- v2t;
      v2s -- v3t;
      v4s -- v3t;
    };

    \graph[use existing nodes, edges = {bidir}] {
      v1t -- v2t;
      v2t -- v3t;
      v3t -- v4t;
      v1t --[bend right] v3t;
    };
    
    \draw[arc] (v1s) to (11.5, 1) to (11.5,-1.75) to (14,-1.75) to (v3t);
    
    \node (v1) at (12,3) {$v_1$};
    \node (v2) at (13,3) {$v_2$};
    \node (v3) at (14,3) {$v_3$};
    \node (v4) at (15,3) {$v_4$};
    
    \graph[use existing nodes, edges = {arc}] {
      v1 -- v2;
      v2 -- v3;
      v4 -- v3;
      v1 --[bend right] v3;
    };
    \node[rotate=270] at (7.5,2) {$\Longrightarrow$};
    \node[rotate=90] at (13.5,2) {$\Longrightarrow$};
    \node at (10.5,0) {$\Longrightarrow$};
    
  \end{tikzpicture}
  \caption{Turing reduction showing that \textsc{transitive closure} can be solved in time
  	$O(T(p))$, given a $T(p)$ algorithm for \textsc{latent projection}.
  The edges in the transitive closure graph (top right) can be read off as edges
  $v_i^s \rightarrow v_j^t$ for $i \neq j$ in the latent projection graph (bottom
  right).}
  \label{fig:tctolp}
\end{figure}
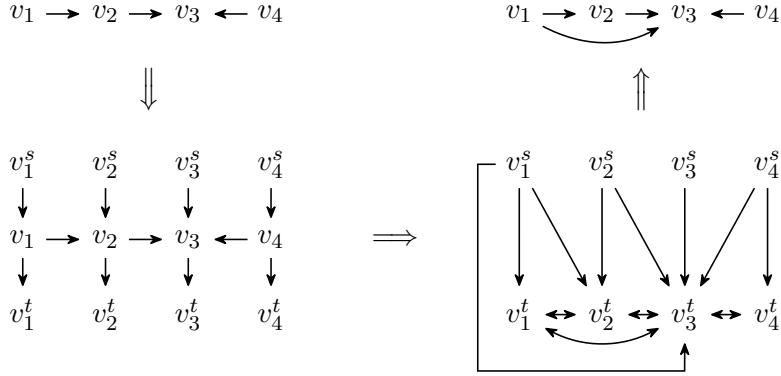

\subsection{Latent Projection}
\label{appendix:latentprojection}
\emph{Latent projection}~\citep{richardson2003markov} is a tool for removing
nodes from an ADMG while preserving the causal and probabilistic information
between the remaining ones.  Originally introduced for removing latent
variables it has also become popular as a method for simplifying graphical
criteria \citep{textor2016robust,witte2020efficient,henckel2024graphical}. The
complexity of computing a latent projection has not been discussed in the
literature. We now show that computing the latent projection is computationally
equivalent to the problem of computing the transitive closure of a directed
graph and, in turn, equivalent to \textsc{bmm} and \textsc{moralize}.
Linear-time algorithms therefore need to avoid latent projections, as we have
done in Section~\ref{subsection:efficientIV}. Formally, the latent projection
of a DAG is defined as follows.

\begin{definition}[Latent projection]
  Let $G = (V \cup L, (E_1,E_2))$ be an ADMG with $V$ denoting observed and
  $L$ latent variables. The latent projection of $G$ over $L$ is the ADMG $G^L = (V,
  (E'_1,E'_2))$ containing an edge $a \tikzrightarrow b$ in $E_1'$ if $G$
  contains a directed path $a \tikzrightarrow \dots \tikzrightarrow b$ with all
  non-endpoint nodes being in $L$ and an edge $a \tikzleftrightarrow b$  in
  $E_2'$ if $G$ contains a path of the form $a \tikzleftarrow \dots
  \tikzleftarrow f \tikzrightarrow \dots \tikzrightarrow b$ with all
  non-endpoint nodes being non-colliders and in $L$.
\end{definition}

\begin{example}
  Let $G$ be the ADMG at the bottom left of Figure~\ref{fig:tctolp}. Suppose
  the nodes in $L=\{v_1,v_2,v_3,v_4\}$ are latent and the nodes in
  $O=\{v_1^s,v_2^s,v_3^s,v_4^s,v_1^t,v_2^t,v_3^t,v_4^t\}$ are observed. The
  latent projection $G^L$ of $G$, is given on the right-hand side. We obtain
  $G^L$ from $G$ as follows. Add a directed edge between two nodes in $G^L$ if
  there is either a directed edge or a directed path consisting of latent
  variables between them in $G$. For example, $G$ contains the path $v_1^s
  \tikzrightarrow v_1 \tikzrightarrow v_2 \tikzrightarrow v_3 \tikzrightarrow
  v_3^t$ so we add the edge $v_1^s \tikzrightarrow v_3^t$ to $G^L$. Similarly,
  we add a bidirected edge between two nodes in $G^L$, if there is either a
  bidirected edge or there exists a non-directed and colliderless path
  consisting of latent nodes between them. For example, $G$ contains the path
  $v_3^t \tikzleftarrow v_3 \tikzleftarrow v_4 \tikzrightarrow v_4^t$ so we add
  the edge $v_3^t \tikzleftrightarrow v_4^t$ to $G^L$. 
\end{example}

We call the computational problem of constructing the latent projection for a
given directed graph \textsc{latent-projection}. We put it in relation to
\textsc{transitive-closure}, the problem of computing the transitive closure of
a directed graph.

\begin{definition}[Transitive closure]
  Let $G$ be a directed graph. The \emph{transitive closure} of $G$ is the
  graph $T$, which contains $a \tikzrightarrow b$ if there is a directed path
  $a \tikzrightarrow \dots \tikzrightarrow b$ in $G$.
\end{definition}

We now show that \textsc{latent-projection} is computationally equivalent to
\textsc{transitive-closure}. Remarkably, we rely on CIfly state-space graphs
for part of the proof (see Figure \ref{figure:latent reduction}); an example of
how they can be useful beyond being a tool for CIfly algorithms.

\begin{theorem}
  Let $G$ be a directed graph. If there exists an $O(T(p))$ algorithm for
  \textsc{latent-projection}, with $T: \mathbb{N} \mapsto \mathbb{N}$, then the
  transitive closure of $G$ can be computed in worst-case time $O(T(p))$.
  Conversely, if there exists an $O(T(p))$ algorithm for
  \textsc{transitive-closure}, then the latent projection of $G$ over any $L$
  can be computed in worst-case time $O(T(p))$. 
\end{theorem}

\begin{figure}[t]
  \begin{minipage}{0.5\textwidth}
    \vspace*{-0.95cm}
    \begin{tabular}{rlrlr}
      \multicolumn{5}{@{}p{2in}@{}}{\textbf{Sets}:   $L$}                                  \\[0.25em] \toprule
      \multicolumn{2}{c}{current}     & \multicolumn{2}{c}{next}  &                        \\ \cmidrule(r){1-2}\cmidrule(l){3-4}
      n.-t.\            & color       & n.-t.\            & color & rule                   \\ \midrule
      $\tikzrightarrow$ & init, lat.\ & $\tikzrightarrow$ & yield & $\text{next} \notin L$ \\
      $\tikzrightarrow$ & init, lat.\ & $\tikzrightarrow$ & lat.\ & $\text{next} \in L$    \\ \bottomrule
    \end{tabular}
  \end{minipage}\hspace*{0.5cm}
  \begin{minipage}{0.5\textwidth}
    \begin{tabular}{rlrlr}
      \multicolumn{5}{@{}p{2in}@{}}{\textbf{Sets}:   $L$}                                   \\ [0.25em] \toprule
      \multicolumn{2}{c}{current} & \multicolumn{2}{c}{next}       &                        \\ \cmidrule(r){1-2}\cmidrule(l){3-4}
      n.-t.\            & color & n.-t.\            & color        & rule                   \\ \midrule
      $\tikzleftarrow$  & init  & $\tikzleftarrow$  & lat.\        & $\text{next} \in L$    \\
      $\tikzleftarrow$  & lat.\ & any               & lat.\        & $\text{next} \in L$    \\
      $\tikzrightarrow$ & lat.\ & $\tikzrightarrow$ & lat.\        & $\text{next} \in L$    \\
      $\tikzrightarrow$ & lat.\ & $\tikzrightarrow$ & \text{yield} & $\text{next} \notin L$ \\ \bottomrule
    \end{tabular}
  \end{minipage}
  \caption{CIfly rule tables for constructing state-space graphs that can be
  used to compute the latent projection of a DAG over $L$ by using
  \textsc{transitive-closure}. }
  \label{figure:latent reduction}
\end{figure}

\begin{proof}
  Assume first we have an algorithm for \textsc{latent-projection}. We aim to
  compute the transitive closure for a directed graph $G = (V, E)$. To do so we
  construct a new directed graph $G' = (V^s \cup V \cup V^t, E')$, with 
  \[
    E' = \{ v_i^s \tikzrightarrow v_i \mid v_i \in V \} \cup \{ v_i
    \tikzrightarrow v_j \mid v_i \tikzrightarrow v_j \in E \} \cup \{ v_i
    \tikzrightarrow v_i^t \mid v_i \in V \}. 
  \]
  and let $L \coloneq V$ be the set of latent variables. If and only if $v_i^s
  \tikzrightarrow v_j^t$ for $i \neq j$ is in the latent projection of $G'$ over $L$, then the
  edge $v_i \tikzrightarrow v_j$ is in the transitive closure of $G$. The
  reason for this is that all paths over non-collider latent variables from
  $v_i^s$ to $v_j^t$ have the form $v_i^s \tikzrightarrow v_i \tikzrightarrow
  \dots \tikzrightarrow v_j \tikzrightarrow v_j^t$ and therefore correspond to
  a path of the form $v_i \tikzrightarrow \dots \tikzrightarrow v_j$ in $G$. We
  illustrate this construction in Figure~\ref{fig:tctolp}. Observe that the
  graph $G'$ has size linear in the size of $G$ and can be constructed in
  linear-time.  Moreover, the size of the projected graph is at most $O(p^2)$,
  which is in $O(T(p))$, as the latent-projection may have worst-case size
  $\Omega(p^2)$ yielding a trivial lower bound on $T(p)$. Hence,
  \textsc{transitive-closure} can be solved in time $O(T(p))$.

 Assume now that we have an algorithm for \textsc{transitive-closure}. We now
 show how \textsc{latent-projection} can be solved with two calls to
 \textsc{transitive-closure}, one determining the directed and the other the
 bidirected edges of the latent projection. The calls to
 \textsc{transitive-closure} are performed on the state-space graphs given by
 the rule tables in Figure~\ref{figure:latent reduction}. The table on the left
 tracks directed paths over latents and the table on the right tracks paths of
 the form $\tikzleftarrow \cdots \tikzleftarrow \tikzrightarrow \cdots
 \tikzrightarrow$ (with at least one $\tikzleftarrow$ and at least one
 $\tikzrightarrow$). We call the transitive closures of these two graphs
 $T_{\mathrm{dir}}$ and $T_{\mathrm{bidir}}$, respectively. An edge $(a,
 \tikzrightarrow, \text{init})$ to $(b, \tikzrightarrow, \text{yield})$ in
 $T_{\mathrm{dir}}$ indicates an edge $a \tikzrightarrow b$ in the latent
 projection of $G$. An edge $(a, \tikzleftarrow, \text{init})$ to $(b,
 \tikzrightarrow, \text{yield})$ in $T_{\mathrm{bidir}}$ corresponds to an edge
 $a \tikzleftrightarrow b$ in the latent projection of $G$. By the definition
 of CIfly reductions, the number of nodes in the state-space graph is in
 $O(p)$. Moreover, they can be constructed in time $O(p + m)$, as shown in
 Theorem~\ref{theorem:lintime}, which is in $O(p^2)$. We note that $T(p)$ is
 trivially in $\Omega(p^2)$ because there exists inputs and outputs of the
 transitive closure of this size. Hence, constructing the two state-space
 graphs is in $O(T(p))$ and the same holds for reading off the edges in the
 computed latent projection yielding overall time $O(T(p))$ of the procedure.
\end{proof}

The tables in Figure \ref{figure:latent reduction} can be used to implement a
latent projection algorithm in CIfly. For both tables start $p$ CIfly calls;
one from each node in $G$. This computes the latent projection in time $O(p
\cdot (p + m))$. There exist ways to solve transitive closure, and therefore
latent projection, with a better worst-case run time using fast matrix
multiplication algorithms, but for practical purposes this is a reasonable
implementation. 

\section{Conclusion and Outlook}
\label{section:conclusion}

We introduce CIfly, a framework for developing efficient algorithms and
algorithmic primitives in causal inference based on causal-to-reach reductions.
By reducing diverse reasoning tasks to reachability in dynamically constructed
state-space graphs, we provide a unified and efficient abstraction for
algorithm design. We propose rule tables as a high-level specification language
and provide a compiled Rust backend to ensure efficient execution, with
bindings for Python and R. This architecture facilitates prototyping,
reproducibility and deployment. It also removes the need to rewrite
core logic for new tasks or languages, a common source of errors. By
construction, any algorithm expressed within the rule table schema is
guaranteed to run in linear time. Through this, CIfly abstracts away low-level
implementation details and lets users focus on the logic of causal reasoning.

We demonstrate CIfly’s performance and flexibility across tasks on DAGs,
CPDAGs, and ADMGs, including novel algorithms for instrumental variables.
Future work includes CIfly algorithms for the front-door criterion
\citep{wienobst2024linear} and extensions to MAGs and PAGs
\citep{richardson2002ancestral}. Other potential avenues are applications to m-graphs \citep{mohan2021graphical}, max-linear Bayesian networks \citep{amendola2022conditional} or extremal graphical models \citep{engelke2020graphical}. Our website at \url{cifly.pages.dev}
includes additional rule tables and implementations and invites community
contributions. The CIfly framework also opens a broader research agenda,
centered around the question of which causal tasks can be expressed as
reachability problems. 

We believe that the flexibility, efficiency and transparency of CIfly can
provide a fruitful basis for future algorithm and software development in
causal inference. What BLAS and LAPACK are to linear algebra, we envision CIfly
becoming for algorithm development in causal reasoning. And, as a welcome
convenience, it now offers clean primitives even for simple tasks such as
descendant computation, that is, no more need to exponentiate adjacency
matrices or hand-roll your own BFS just to answer basic graph queries.

\DeclareRobustCommand{\VAN}[3]{#3}
\bibliography{main}

\clearpage

\appendix

\section{A Sound and Complete Algorithm for Finding Valid Conditional Instrumental Sets}
\label{appendix:iv:sound:and:complete:finder}
In this section, we derive an algorithm that decides whether a valid
conditional instrumental set relative to $(x,y)$ exists, respectively finds
such a set, based on the sound and complete graphical validity criterion by
\citet{henckel2024graphical}. This procedure is based on the polynomial-time
algorithm given by \citet{van2020algorithmics}. To ease notation we will use
$\tilde{G}$ to denote the ADMG obtained from $G$ by deleting all edges from $x$
to nodes in $\causal_G(x, y)$ as defined in Theorem \ref{theorem:validcis}. The
proofs are  minor adaptations of those by \citet{van2020algorithmics}. We
proceed as follows:

\begin{enumerate}
  \item In Definition \ref{definition: nearest separator}, we define the concept
    of a nearest separator.
  \item In Lemma \ref{lemma: nearest separator suffices}, we use these results
    to show that if there exists a valid conditional instrumental set $(z,W)$
    relative to $(x,y)$ in an ADMG $G$ such that $W\subseteq
    A=\an_{\tilde{G}}(\{y, z\})$, then the nearest separator $W'$ restricted to
    nodes not in $A\setminus \forb_{G}(x, y)$ also forms a valid conditional
    instrumental set.
  \item In Lemma \ref{lemma: ancestral IV}, we show that if a valid conditional
    instrumental set relative to $(x,y)$ in $G$ exists, then there also exists
    one satisfying the restriction that $W\subseteq \mathrm{an}(z\cup y,
    \tilde{G})$.
  \item Based on these results, in Algorithm~\ref{algorithm:condiv}, we show
    how to compute, for a given ADMG $G$ and $(x, y)$, a conditional instrument
    $(z, W)$ or decide that none exists in time $O(p \cdot (p + m))$. 
\end{enumerate}

\begin{definition}[Nearest separator]
  Consider disjoint node sets $\{y, z\}$ and $R$ in an ADMG $G$. A node set
  $W$ is a nearest separator with respect to $y$, $z$ and $R$ if
    \begin{enumerate}
      \item $W \subseteq R \cap \an_G(\{y, z\})$,
      \item $y \perp_{G} z \mid W$ and
      \item for all $w\in W$ and any set $W' \subseteq (R \cap \an_G(\{y, z\}))
        \setminus \{x,y,w\}$ such that $w \not\perp_{G} z \mid W'$ it holds
        that $y \not\perp_{G} z \mid W'$.
    \end{enumerate}
    \label{definition: nearest separator}
\end{definition}

\begin{lemma}[Checking nearest separators suffice]
  Consider nodes $x$, $y$ and $z$ in an ADMG $G$. There exists a set $W
  \subseteq A= \mathrm{an}(y \cup z,\tilde{G})$, such that $(z,W)$ is a valid
  conditional instrumental set relative to $(x,y)$ in $G$ if and only if there
  exists a nearest separator $W'$ relative to $(y,z,A \setminus \forb_{G}(x,
  y))$ in $\tilde{G}$, and any such $W'$ satisfies that $(z,W')$ is a valid
  conditional instrumental set relative to $(x,y)$ in $G$.
  \label{lemma: nearest separator suffices}
\end{lemma}

\begin{proof}
  The if statement is trivially true so it suffices to show that if there
  exists a $W \subseteq A$, such that $(z,W)$ is a valid conditional
  instrumental set relative to $(x,y)$ in $G$ then a nearest separator $W'$
  relative to $(y,z,A \setminus \forb_G(x, y))$ in $\tilde{G}$ exists and
  $(z,W')$ is a valid conditional instrumental set relative to $(x,y)$ in $G$.
  By assumption, $y \perp_{\tilde{G}} z \mid W$, $x \not\perp_{G} z \mid W$ and
  $W \subseteq A \setminus \forb_G(x, y)$. Using the fact that a nearest
  separator $W'$ exists if some separating set $W$ exists (Proposition 6.7
  by~\cite{van2020finding}), the first and the third statement imply that there
  exists a nearest separator $W'$ relative to $(y,z,A \setminus \forb_G(x, y))$
  in $\tilde{G}$. By nature of being a nearest separator, $y \perp_{\tilde{G}}
  z \mid W'$ and clearly $W' \subseteq A \setminus \forb_G(x, y)$. It therefore
  only remains to show that $x \not\perp_{G} z \mid W'$.
  
  We first note that since $W \cap \forb_G(x, y) = \emptyset$ any walk from $x$
  to $z$ that starts with one of the edges deleted in $\tilde{G}$ must contain
  a collider in $\forb_G(x, y)$. As descendants of forbidden nodes (that are
  not $x$) are forbidden it follows that the walk is blocked. Therefore, there
  exists a walk not containing any such edge and it follows that $x
  \not\perp_{\tilde{G}} z \mid W$. We can therefore restrict our attention to
  $\tilde{G}$ for the remainder of the proof.
  
  Consider a walk $u$ from $z$ to $x$ open given $W$. If $u$ is blocked by $W
  \cup W'$ it follows that it must contain a non-collider in $W'$. Let $w' \in
  W'$ be the non-collider closest to $z$ on $u$. Then the subwalk from $z$ to
  $w'$ is open given $W$ and therefore it follows that $y \not\perp z \mid W$
  since $W'$ is a nearest separator. For the same reason $x \notin W'$ since by
  assumption $x \not\perp_{\tilde{G}} z \mid W$. Therefore, any such $u$ is
  open given $W \cup W'$ and $x \notin W'$.
  
  We have established $u$ is open given $W \cup W'$ so assume it is blocked by
  $W'$. This implies that no non-collider on $u$ is in $W'$ and that there is
  at least one collider not in $W'$ (but in $W$). Since $W \subseteq A$ for any
  such collider there exists a directed path to either $z$ or $y$. Expanding
  $u$ by adding these directed paths and back, choosing the $z$ furthest right
  on this new path we obtain a subwalk from either $z$ to $y$ that is open
  given $W'$ or from $z$ to $x$ that is open given $W'$. Both yield a
  contradiction. Therefore, $u$ cannot be blocked by $W'$ and it follows that
  $x \not\perp_{\tilde{G}} z \mid W'$, which trivially implies that  $x
  \not\perp_{G} z \mid W'$ since $G$ is a supergraph of $\tilde{G}$.
\end{proof}

\begin{table}[t]
  \centering
  \begin{tabular}{rrrrr}
    \multicolumn{3}{@{}p{2in}@{}}{\textbf{Sets}:   $X, Z, A$}                                                               \\
    \multicolumn{3}{@{}p{2in}@{}}{\textbf{Start}:  $(X, \tikzleftarrow)$}                                                   \\
    \multicolumn{3}{@{}p{2in}@{}}{\textbf{Return}: any}                                                                     \\ [0.25em] \toprule
    current                                  & next             &                                                           \\ \midrule
    $\tikzrightarrow$, $\tikzleftrightarrow$ & $\tikzleftarrow,\tikzleftrightarrow$ & $\text{next} \in A$                                       \\
    any                                      & any              & $\text{next} \in A \text{ and } \text{current} \not\in Z$ \\ \bottomrule
  \end{tabular}
  \caption{CIfly rule table for computing $\closure_G(X, Z, A)$ in an ADMG. }
  \label{figure:closure:ruletable}
\end{table}

\begin{lemma}[Checking ancestral instruments suffice]
  Consider nodes $x$ and $y$ in an ADMG $G$. There exists a valid conditional
  instrumental set $(z,W)$ relative to $(x,y)$ in $G$ if an only if there
  exists a node $z'$ and a set $W' \subseteq A= \mathrm{an}(y \cup
  z,\tilde{G})$, such that $(z',W')$ is a valid conditional instrumental set
  relative to $(x,y)$ in $G$. \label{lemma: ancestral IV}
\end{lemma}

\begin{proof}
  The one direction is trivial, so we only need to show that if there exists a
  valid conditional instrumental set $(z,W)$ relative to $(x,y)$ in $G$ then
  there exists a node $z'$ and a set $W' \subseteq A$, such that $(z',W')$ is a
  valid conditional instrumental set relative to $(x,y)$ in $G$. Let $(z,W)$ be
  a valid conditional instrumental set and $u$ a walk from $z$ to $x$ open
  given $W$. By the same argument as in the proof of Lemma \ref{lemma: nearest
  separator suffices}, $u$ is also an open walk in $\tilde{G}$ and we can
  restrict ourselves to considering $\tilde{G}$.
  
  Choose $z,W$ and $u$ such that $u$ has the minimum number of colliders
  possible and $W$ is of minimum size, that is, any other walk would have at
  least as many colliders as $u$, and any walk with the same number of
  colliders is opened by a set with at least as many nodes as $W$. By
  construction no node $w \in W$ can be removed without opening a walk from $z$
  to $y$ given $W$ in $\tilde{G}$ or blocking $u$. In the former case $w$ is a
  non-collider on some walk from $z$ to $y$ open given $W\setminus w$ and
  therefore, $w \in \an(y \cup z \cup (W \setminus w),\tilde{G})$.
  
  In the latter case, $w$ must be collider on $u$ and the subwalk of $u$ from
  $x$ to $w$ is open given $W \setminus w$ and contains at least one less
  collider than $u$. This implies that there exists a walk $u'$ from $w$ to $y$
  that is open given $W \setminus w$ as otherwise $(w,W\setminus w)$ would be a
  valid conditional instrumental set relative to $(x,y)$ in $G$, contradicting
  the assumption that $u$ has the fewest number of colliders possible. The path
  $u'$ must begin with an edge of the form $w \rightarrow$. Otherwise $w$ is a
  collider on $u'' = u(x,w) \oplus u'$ and therefore $u''$ is a walk from $x$
  to $y$ that is open given $W$. Therefore, either $w \in \mathrm{an}(y,
  \tilde{G})$ or there exists a collider on $u'$ in $W \setminus w$ and
  therefore $w \in \an(W \setminus w, \tilde{G})$.
  
  We have shown that every node $w \in W$ satisfies that $w \in \an(y \cup z
  \cup (W \setminus w),\tilde{G})$. By the acyclicity of $\tilde{G}$ it follows
  that $W \subseteq A$.
\end{proof}

\begin{algorithm}[t]
  \DontPrintSemicolon
  \SetKwInOut{Input}{input}\SetKwInOut{Output}{output}
  
  \BlankLine
  \Input{A graph $G = (V, (E_{\text{dir}}, E_{\text{bidir}}))$, treatment $x$ and outcome $y$.}
  \Output{Set of nodes $R$.}
  \SetKwFunction{FNearestSep}{nearestSeparator}
  \SetKwProg{Fn}{function}{}{end}

  \BlankLine
  \Fn{\FNearestSep{$G, x, y, R$}}{
    $Z_0 \coloneq R \cap \an_G(x, y) \setminus \{x, y\}$ \;
    $X^* \coloneq \closure_G(x, Z_0, A)$ \;
    \uIf{$y \in X^*$}{
      \Return $\bot$ \;
    }
    \Else{
      \Return $Z_0 \cap X^*$ \;
    }
  }

  \BlankLine
  \ForEach{$z \in V \setminus (\forb_G(x, y) \cup \{y\})$}{
    $W \coloneq$ \FNearestSep{$\tilde{G}$, $y$, $z$, $V \setminus \forb_G(x, y)$} \;
    \If{$W \neq \bot$ \textbf{\emph{and}} $x \perp_G z \mid W$ \label{algorithm:condiv:check:line}}{
      \Return $(z, W)$ \; \label{algorithm:condiv:return:line}
    }
  }

  \Return $\bot$ \;
  \caption{$O(p \cdot (p + m))$ algorithm for finding a valid conditional
  instrument relative to $(x, y)$ in an ADMG $G$ or deciding that none exists.}
  \label{algorithm:condiv}
\end{algorithm}

Algorithm~\ref{algorithm:condiv} describes how the procedure suggested by these
proofs can be implemented. For every $z \not\in \forb_G(x, y)$, the algorithm
computes a nearest separator $W$ that does not contain a node in $\forb_G(x,
y)$ and afterwards checks whether $(z, W)$ is a valid conditional instrumental
set. If this is never the case then by Lemma \ref{lemma: nearest separator
suffices} and Lemma \ref{lemma: ancestral IV} no valid conditional instrumental
set relative to $(x,y)$ in $G$ exists. 

We compute the nearest separators using the linear-time algorithm developed by
\citet{van2020algorithmics}, which can be directly expressed in CIfly. For this
the notion of a \emph{closure} is central. Here, $\closure_G(X, Z, A)$ is
defined as the set of all nodes $v$ such that there exists a path from $X$ to
$v$ that only contains nodes in $A$ and no non-collider in
$Z$~\citep{van2020finding}. A CIfly rule table for computing the closure is
given in Table~\ref{figure:closure:ruletable}. The remaining graphical notions
in Algorithm~\ref{algorithm:condiv}, such as ancestors, forbidden nodes or
$\tilde{G}$ (the graph constructed by removing edges between $x$ and nodes in
$\causal_G(x, y)$) can be computed in linear time relying on straightforward
rule tables, many of which we have already introduced in other contexts in this
paper. Overall, this yields an $O(p \cdot (p+m))$ algorithm as linear time is
needed for every candidate $z$. 

\begin{theorem}\label{theorem:sound:and:complete:iv:algo}
  Let $x$ and $y$ be distinct nodes in an ADMG $G$. Then,
  Algorithm~\ref{algorithm:condiv} returns a valid conditional instrument $(z,
  W)$, if there exists one, and $\bot$ else in time $O(p \cdot (p + m))$. 
\end{theorem}

\begin{proof}
  By Lemma~\ref{lemma: nearest separator suffices} and~\ref{lemma: ancestral
  IV}, it holds that there exists a valid conditional instrument $(z, W)$ if,
  and only if, the pair $(z, W')$ for a nearest separator $W'$ relative to $(y,
  z, A \setminus \forb_G(x, y))$ is also a valid conditional instrument. As the
  algorithm tests all non-forbidden $z$, the nearest separator is computed by
  the (unmodified) algorithm given by~\cite{van2020finding} and it is checked
  whether $(z, W')$ fulfills the conditions of a conditional instrument
  (condition~\ref{condition:validcis:forb} is satisfied by construction of $z$
  and $W'$, condition~\ref{condition:validcis:ztox} is checked explicitly in
  line~\ref{algorithm:condiv:check:line} and
  condition~\ref{condition:validcis:znottoy} is satisfied by the definition of
  $W'$), correctness follows. For the run time, observe that for each $z$, a
  constant amount of set operations and reachability calls are performed, thus
  yielding time $O(p \cdot (p+m))$. 
\end{proof}

We can adapt the algorithm to find a valid conditional instrumental set $(z,
W)$ for \emph{all} $z$ for which such a $W$ exists, simply by not terminating
early in line~\ref{algorithm:condiv:return:line} and storing a list of all
found instruments. For a fair comparison with the algorithm for finding
conditional instrumental sets by \texttt{DAGitty} which employs this general
strategy, we use this adapted version of Algorithm \ref{algorithm:condiv} in
our simulations. 

\section{Proofs for Section~\ref{section:iv}}
\label{appendix:proofIV}

\ValidCISwalks*

\begin{proof}
  As \citet{henckel2024graphical} have proven Theorem~\ref{theorem:validcis},
  it suffices to show that tuple $(Z,W)$ satisfies the graphical conditions in
  Theorem \ref{theorem:validcis} if and only if it also satisfies those in
  Theorem \ref{theorem:walksIV}. We first prove that any tuple $(Z,W)$ that
  satisfies the condition in Theorem~\ref{theorem:validcis} also satisfies
  those in Theorem~\ref{theorem:walksIV}. By the characterization of
  d-separation via walks~\citep{shachter1998bayes}, $x \not\perp_{G} Z \mid W$
  implies that Condition~\ref{condition:powerIV} in Theorem
  \ref{theorem:walksIV} holds. Furthermore, $\causal_G(x, y) \subseteq
  \forb_G(x, y)$ and therefore $(Z \cup W) \cap \forb_G(x, y) = \emptyset$
  implies that Condition~\ref{condition:causalblockedIV} holds. To show that
  Condition~\ref{condition:noncausalIV} holds, let $w$ be a proper walk from
  $Z$ to $y$ and suppose it is open given $W$. Since, $Z \perp_{\tilde{G}} y
  \mid W$ it must contain an edge $x \tikzrightarrow c$ where $c \in
  \causal_G(x, y)$. Consider the subwalk $w'=w(x,y)$. Any collider on $w'$
  needs to be in $W$ but the first collider is also a descendant of the causal
  node $c$ and therefore forbidden. This violates $(Z \cup W) \cap \forb_G(x,
  y) = \emptyset$ and therefore $w'$ cannot contain a collider and must be
  directed towards $y$. Therefore, any proper walk from $Z$ to $y$ is either
  blocked by $W$ or satisfies the exception in Condition
  \ref{condition:noncausalIV}.
  
  We now prove that any tuple $(Z,W)$ that does not satisfy the condition in
  Theorem \ref{theorem:validcis} also does not satisfy those in Theorem
  \ref{theorem:walksIV}. By the characterization of d-separation via walks
  \citep{shachter1998bayes}, $x \perp_{G} Z \mid W$ implies that Condition
  \ref{condition:powerIV} in Theorem \ref{theorem:walksIV} is violated. By the
  same characterization, $Z \not\perp_{\tilde{G}} y \mid W$ implies that there
  exists a proper walk $w$ from $Z$ to $y$ open given $W$ in $\tilde{G}$. As
  $\tilde{G}$ is a subgraph of $G$ it follows that $w$ is also a proper walk
  from $Z$ to $y$ open given $W$ in $G$. The walk $w$ also cannot end with a
  segment of the form $x \tikzrightarrow \dots \tikzrightarrow y$ as it would
  then contain an edge that does not exist in $\tilde{G}$. It follows that the
  existence of $w$ violates Condition \ref{condition:noncausalIV}.
  
  It therefore only remains to show that if $(Z \cup W) \cap \forb_G(x, y) \neq
  \emptyset$ then $(Z,W)$ may not satisfy the conditions of Theorem
  \ref{theorem:walksIV}, where we can assume that $(Z \cup W)$ satisfies $x
  \not\perp_{G} Z \mid W$ and $Z \perp_{\tilde{G}} y \mid W$. Suppose first
  that $Z \cap \forb_G(x, y) \neq \emptyset$. This implies that there exists a
  trek, that is, a walk of the form $Z \tikzleftarrow \dots \tikzleftarrow c
  \tikzrightarrow \dots \tikzrightarrow y$ where $c \in \causal_G(x, y)$ with
  possibly $c \in Z$. As by assumption $Z \perp_{\tilde{G}} y \mid W$ it
  follows that the trek must at least contain one node in $W$ and therefore $W
  \cap \forb_G(x, y) \neq \emptyset$. It therefore suffices to consider the
  case that $W \cap \forb_G(x, y) \neq \emptyset$. There are two cases here.
  First, $W \cap \causal_G(x, y) \neq \emptyset$ and second, $W \cap \forb_G(x,
  y) \setminus \causal_G(x, y) \neq \emptyset$. The former is a violation of
  Condition \ref{condition:causalblockedIV} so we can assume only the latter is
  true. The fact that $W \cap (\forb_G(x, y) \setminus \causal_G(x, y)) \neq
  \emptyset$ and $W \cap \causal_G(x, y) = \emptyset$ implies that we can
  choose a node $v$ from $W$ such that there exist walks of the form $x
  \tikzrightarrow \dots \tikzrightarrow v$ and $v \tikzleftarrow \dots
  \tikzleftarrow c \tikzrightarrow \dots \tikzrightarrow y$ where $v \neq c \in
  \causal_G(x, y)$ and such that both walks do not contain any other node in
  $W$. Combining them we obtain a walk $w$ from $x$ to $y$ that begins with an
  edge out of $x$, is open given $W$ and does not end with a segment of the
  form $x \tikzrightarrow \dots \tikzrightarrow y$. Furthermore, as $x
  \not\perp_{G} Z \mid W$ we can choose a proper walk $w'$ from $Z$ to $x$ that
  is open given $W$. Consider the walk, $w''=w \oplus w'$. By construction $x$
  is a non-collider on $w''$ and therefore $w''$ is a walk from $Z$ to $y$ that
  is open given $W$. By nature of $w'$ it can also not end with a segment of
  the form $x \tikzrightarrow \dots \tikzrightarrow y$. Therefore either $w''$
  or a subwalk violate Condition \ref{condition:noncausalIV}.
\end{proof}

\ValidCISalgotheorem*

\begin{proof}
  To prove the statement, we need to show that the set
  $O_{\ref{condition:causalblockedIV}}$ that Table~\ref{table:causalblocked}
  computes contains all nodes $y$ violating
  Condition~\ref{condition:causalblockedIV} and the set
  $O_{\ref{condition:noncausalIV}}$ computed by Table~\ref{table:noncausal}
  contains all nodes $y$ violating Condition~\ref{condition:noncausalIV}. We
  begin with set $O_{\ref{condition:causalblockedIV}}$. Clearly, every path
  violating Condition~\ref{condition:causalblockedIV} has the form $x
  \tikzrightarrow a_1 \tikzrightarrow \cdots \tikzrightarrow a_k
  \tikzrightarrow w \tikzrightarrow b_1 \tikzrightarrow \cdots \tikzrightarrow
  b_{\ell}$ where $k = 0$ and $\ell = 0$ are possible, $a_i \not\in W$ and $w
  \in W$. Table~\ref{table:causalblocked} checks precisely paths of this form.
  For set $O_{\ref{condition:noncausalIV}}$, observe that the color non-causal
  tracks all open walks from $Z$ using the standard d-separation rules with the
  exception that the color causal-end is reached for nodes on an open path from
  $Z$ which ends with $x \tikzrightarrow \cdots \tikzrightarrow$.
  
  The task of finding all $y$ such that $(Z,W)$ is a valid conditional
  instrumental set relative to $(x,y)$ in $G$ can thus be solved with three
  CIfly calls, which take $O(p+m)$ time by Theorem~\ref{theorem:lintime}. The
  output is computed by a constant number of set union and set difference
  operations, which can be implemented in $O(p)$.
\end{proof}

\OptimalCISnewtheorem*

\begin{proof}
  Denote $D \setminus \{x, y\}$ by $D_s$. Since $y$ is the only possible node
  in $\de_{G}(D_s) \setminus D_s$, it follows that all bidirected edges in $G$
  between nodes in $V \setminus D_s$ also exist in $\tilde{G}$ and vice versa.
  Furthermore, any edge of the form $c \tikzleftrightarrow y$ in $\tilde{G}$
  corresponds to a path of the form $c \tikzleftrightarrow d_1 \tikzrightarrow
  \dots \tikzrightarrow d_l \tikzrightarrow y$ with possibly $l=0$ in $G$.
  Finally, any path of the form $p \tikzrightarrow y$ in $\tilde{G}$
  corresponds to a path of the form $p \tikzrightarrow d_1 \tikzrightarrow
  \dots \tikzrightarrow d_l \tikzrightarrow y$ with possibly $l=0$ in $G$. It
  follows that any node lying on one of the paths defining $\zoptnew$ in $G$
  that is not in $D_s$ lies on one of the paths defining  $\zopt$ in
  $\tilde{G}$. Since both sets are defined to consist of nodes not in $D_s$
  this suffices to show they are the same set.
  
  Regarding, $\woptnew$ and $\wopt$, all bidirected edges in $G$ also exist in
  $\tilde{G}$ and vice versa and the paths consist of nodes not in $D_s$ it
  follows that any of the paths defining $\woptnew$ in $G$ also exists in
  $\tilde{G}$. It follows that the two sets coincide.
\end{proof}

\OptimalCISalgotheorem*

\begin{proof}
  The algorithm directly finds $\woptnew$ and $\zoptnew$ as stated in
  Theorem~\ref{theorem:optimalCISnew} using Table~\ref{figure:optimalIVtables}.
  The transition rules track paths of the form $p \tikzrightarrow a_1
  \tikzleftrightarrow \cdots \tikzleftrightarrow  a_k \tikzleftrightarrow b_1
  \tikzrightarrow \cdots \tikzrightarrow b_{l} \tikzrightarrow y$ with $p, a_1,
  \dots, a_k \in A$ (possibly with $k=0$) and $b_1, \dots b_l \in B$ (again
  possibly with $l=0$). By setting $A \coloneq A$ and $B \coloneq D$, one
  obtains $\woptnew$, and with $A \coloneq A$ and $B \coloneq \emptyset$ and a
  subsequent set difference with $\woptnew$, the set $\zoptnew$ can be
  computed.
\end{proof}

\section{Simulations for Finding Valid Conditional Instrumental Sets}
\label{appendix:iv:simulations}
It is an important algorithmic task to find, given $x,y$ and $G$, a valid
conditional instrumental set $(Z,W)$ relative to $(x,y)$ in $G$. The popular
\texttt{DAGitty} software package~\citep{textor2016robust} provides the
\texttt{instrumentalVariables} function for this purpose. Given $x,y$ and $G$
it returns a list of valid conditional instrumental sets. More precisely, for
every $z$, for which the function finds a $W$ such that $(z,W)$ is a
conditional instrumental set, the function returns that pair $(z,W)$. The
algorithm has complexity $O(p \cdot (p + m))$, as it considers each $z$
separately and finding an appropriate $W$ takes linear time. However, we will
show that the algorithm by \texttt{DAGitty} is not based on a \emph{sound}
validity criterion and thus, at times, produces tuples
$(z,W)$ that are invalid. 

In this section, we compare the \texttt{DAGitty} implementation in terms of
accuracy and run time to two alternative algorithms for finding valid
conditional instrumental sets. First, the sound and complete algorithm
developed in Appendix \ref{appendix:iv:sound:and:complete:finder}, for which we
provide an R and a Python implementation using CIfly. It runs in time $O(p
\cdot (p + m))$ as, again, we find an appropriate $W$ for every potential $z$
separately. Second, the algorithm developed in Section
\ref{subsection:efficientIV} for computing the graphically optimal conditional
instrumental set (see also Figure~\ref{figure:optimalIVcode}). The underlying
criterion, given in Theorem \ref{theorem:optimalcis}, is \emph{not} complete as
there are two conditions which we need to check in turn to ensure that the
output is (i) valid and (ii) graphically optimal. In our implementation we only
check whether the first condition holds as we are solely interested in finding
valid conditional instrumental sets and not in efficiency guarantees. This
algorithm runs in time $O(p+m)$. In the remainder of this section we will refer
to these three algorithms for short as \texttt{DAGitty}, \texttt{complete} and
\texttt{optimal}, respectively.

\begin{table}[t]
  \centering
  \begin{tabular}{crrrrrrrrr} \toprule
         & \multicolumn{3}{c}{\texttt{DAGitty}} & \multicolumn{3}{c}{\texttt{complete}} & \multicolumn{3}{c}{\texttt{optimal}} \\ \cmidrule(r){2-4}\cmidrule(lr){5-7} \cmidrule(l){8-10}
    $p$  & sound   & unsound & none             & sound   & unsound & none              & sound   & unsound & none             \\ \midrule
    $10$ & $0.571$ & $0.005$ & $0.424$          & $0.574$ & $0.000$ & $0.426$           & $0.333$ & $0.000$ & $0.667$          \\
    $20$ & $0.745$ & $0.045$ & $0.210$          & $0.780$ & $0.000$ & $0.220$           & $0.488$ & $0.000$ & $0.512$          \\
    $30$ & $0.785$ & $0.096$ & $0.119$          & $0.870$ & $0.000$ & $0.130$           & $0.511$ & $0.000$ & $0.449$          \\
    $40$ & $0.794$ & $0.114$ & $0.092$          & $0.906$ & $0.000$ & $0.094$           & $0.589$ & $0.000$ & $0.411$          \\
    $50$ & $0.749$ & $0.183$ & $0.068$          & $0.922$ & $0.000$ & $0.078$           & $0.610$ & $0.000$ & $0.390$          \\ \bottomrule
  \end{tabular} \hspace{0.75cm}
  \caption{Proportion of instances in which the algorithms \texttt{DAGitty},
  \texttt{complete}, and \texttt{optimal} returned a list of valid conditional
  instrumental sets (sound), a list containing at least one invalid conditional
  instrumental set (unsound) and an empty list (none), respectively.}
  \label{table:optimalIVsimulations}
\end{table}

\subsection{Simulation Setup}
We consider graphs with $10$, $20$, $30$, $40$ and $50$ nodes. For each
setting, we generate $1000$ instances. The graphs are generated as follows: we
sample two undirected graphs in the Erdős-Renyi model. One with average degree
$1$ and the other with average degree $3$. In the former graph, we subsequently
make all edges bidirected and in the latter graph, we orient the edges into
directed ones according to a uniformly drawn linear ordering of the nodes.
Afterwards, we merge both graphs into a single ADMG. Here, we run two separate
versions for our simulation. In the first, we repeatedly draw distinct $x$ and
$y$ uniformly at random until we obtain a pair such that $y$ is a descendant of
$x$ (we restrict our simulations to graphs with at least one directed edge,
hence, such $x$ and $y$ always exist). In the second, we simply draw distinct
$x$ and $y$ uniformly at random.

For each instance we then use the algorithms \texttt{DAGitty},
\texttt{complete} and \texttt{optimal} to produce lists of conditional
instrumental sets. With the simpler algorithm for verifying validity sketched
in Section \ref{subsection:validIV} we check for each tuple whether it is
indeed valid and report three possible results. If the list is empty, that is,
no valid conditional instrumental set was found, we report \emph{none}. If the
list is not empty and all tuples on the list are valid, we report \emph{sound}.
Otherwise, we report \emph{unsound}. In our simulations, we consider DAGitty version
\texttt{0.3.4} and run the programs on a single thread of a
Intel\textregistered{} Core\textsuperscript{TM} i5-8350 CPU.

\subsection{Simulation Results}
Consider first the simulation in which we enforce that $y$ is a descendant of
$x$ (else the causal effect is trivially zero). We report the fraction of
instances for which the respective algorithms returned \emph{sound},
\emph{unsound} or \emph{no} conditional instrumental set in Table
\ref{table:optimalIVsimulations} over the $1000$ instances we generated for
each graph size. The results show that there are instances where
\texttt{complete} finds valid conditional instrumental sets and the other two
algorithms fail to do so. For \texttt{optimal} this is the case because it is
not complete. For $p=50$, for example, \texttt{optimal} fails to find a valid
conditional instrumental set even though one exists, as certified by
\texttt{complete}, in up to a third of instances. This indicates that more
research into graphically optimal conditional instrumental sets may be
warranted. For \texttt{DAGitty}, on the other hand, this is the case because it
is unsound. For $p=50$, for example, \texttt{DAGitty} returns invalid
conditional instrumental sets in $18\%$ of instances. In
Appendix~\ref{appendix:iv:DAGitty:bug}, we discuss the failure case of
\texttt{DAGitty} in more detail. We also report average run times in
Table~\ref{table:simulation:condiv:runtimes}. We see that the \texttt{complete}
R implementation is approximately 5 times faster than the \texttt{DAGitty} R
implementation. The R implementation of \texttt{optimal} runs in linear time
and is thus even faster by a factor 10. All implementations, however, are fast
enough for small scale causal inference tasks.

\begin{table}
  \centering
  \begin{tabular}{cccc} \toprule
    $p$  & \texttt{DAGitty} & \texttt{complete} & \texttt{optimal} \\ \midrule
    $10$ & $0.0064$         & $0.0006$          & $0.0001$         \\
    $20$ & $0.0090$         & $0.0013$          & $0.0001$         \\
    $30$ & $0.0122$         & $0.0021$          & $0.0001$         \\
    $40$ & $0.0167$         & $0.0032$          & $0.0001$         \\
    $50$ & $0.0203$         & $0.0040$          & $0.0002$         \\ \bottomrule
  \end{tabular}
  \caption{Average run time in seconds for the algorithms \texttt{DAGitty},
  \texttt{complete}, and \texttt{optimal}.}
  \label{table:simulation:condiv:runtimes}
\end{table}

Consider now the simulation in which we do not impose that $y$ be a descendant
of $x$. While this is a reasonable restriction, as else the causal effect is
trivially zero and no conditional instrumental set is needed, we report results
when this restriction is lifted in
Table~\ref{table:optimalIVsimulations:nodesc}. In the generated instances it
may still happen that $y$ is a descendant of $x$, it is just not enforced. As
our construction of the graphically optimal tuple assumes that $y$ is a
descendant of $x$ (see Theorem \ref{theorem:optimalCISnew}), we only apply
\texttt{optimal} to cases where this is true. Despite being of limited
practical importance, the results show that the \texttt{DAGitty} implementation
is not complete if $y$ is not a descendant of $x$, as \texttt{complete} finds
valid conditional instrumental sets in cases where the former fails to do so.

\begin{table}
  \centering
  \begin{tabular}{crrrrrrrrr} \toprule
         & \multicolumn{3}{c}{\texttt{DAGitty}} & \multicolumn{3}{c}{\texttt{complete}} & \multicolumn{3}{c}{\texttt{optimal}} \\ \cmidrule(r){2-4}\cmidrule(lr){5-7} \cmidrule(l){8-10}
    $p$  & sound   & unsound & none             & sound   & unsound & none              & sound   & unsound & none             \\ \midrule
    $10$ & $0.735$ & $0.003$ & $0.262$          & $0.757$ & $0.000$ & $0.243$           & $0.089$ & $0.000$ & $0.911$          \\
    $20$ & $0.904$ & $0.012$ & $0.084$          & $0.924$ & $0.000$ & $0.076$           & $0.075$ & $0.000$ & $0.925$          \\
    $30$ & $0.924$ & $0.014$ & $0.062$          & $0.950$ & $0.000$ & $0.050$           & $0.059$ & $0.000$ & $0.941$          \\
    $40$ & $0.953$ & $0.007$ & $0.040$          & $0.967$ & $0.000$ & $0.033$           & $0.059$ & $0.000$ & $0.941$          \\
    $50$ & $0.955$ & $0.012$ & $0.033$          & $0.976$ & $0.000$ & $0.024$           & $0.050$ & $0.000$ & $0.950$          \\ \bottomrule
  \end{tabular} \hspace{0.75cm}
  \caption{Proportion of instances in which the algorithms \texttt{DAGitty},
  \texttt{complete}, and \texttt{optimal} returned a list of valid conditional
  instrumental sets (sound), a list containing at least one invalid conditional
  instrumental set (unsound) and an empty list (none), respectively, in the
  simulation study where we do not enforce that $y$ is a descendant of $x$.}
  \label{table:optimalIVsimulations:nodesc}
\end{table}

\subsection{Unsoundness of the DAGitty implementation}
\label{appendix:iv:DAGitty:bug}
Our simulations suggest that the algorithm provided by \texttt{DAGitty} to find
valid conditional instrumental sets returns sets violating the sound and
complete graphical criterion by \citet{henckel2024graphical}. This indicates
that the algorithm by \texttt{DAGitty} is not sound in the sense that it
returns invalid conditional instrumental sets. In this section, we provide a
concrete example confirming that this is indeed the case. Before doing so, we
first give a cursory introduction to some notation for covariance matrices, the
two-stage least squares estimator, linear structural equation models compatible
with ADMGs and total causal effects. For more details, we refer readers to
\citet{henckel2024graphical}.

Regarding the covariance matrix notation, consider three random vectors $S,T$
and $W$. We use $\Sigma_{ST}$ to denote the covariance matrix between $S$ and
$T$. Furthermore, we use $\Sigma_{ST.W}=\Sigma_{ST} - \Sigma_{SW}
\Sigma_{WW}^{-1} \Sigma_{WT}$ to denote the conditional covariance matrix
between $S$ and $T$ given $W$. Regarding the two-stage least squares
estimators, consider now four random vectors $x,y,Z$ and $W$ where the first
two are univariate. Let $S_n, T_n$ and $Y_n$ be the random matrices
corresponding to $n$ i.i.d. observations of the random vectors $S =
(x,W),T=(Z,W)$ and $y$. Then, the two-stage least squares estimator
$\hat{\tau}_{yx}^{Z.W}$ is defined as the first entry of the estimator $$
\hat{\gamma}_{ys.t}=Y_n T_n^\top (T_n T_n^\top)^{-1} T_n S_n^\top \{S_n
T_n^\top (T_n T_n^\top)^{-1} T_n S_n^\top\}^{-1}. $$ If $\Sigma_{zx.w} \neq 0$,
then $\hat{\tau}_{yx}^{Z.W}$ converges in probability to the limit
$\tau_{yx}^{Z.W} = \Sigma_{yZ.W} \Sigma_{Zx.W}^{-1}$.

Finally, a linear structural equation model is a set of recursive generating
equations for a random vector $V$ that can be written as $V = \mathcal{A} V +
\epsilon$, where $\mathcal{A}$ is a lower-triangular adjacency matrix and
$\epsilon$ a, typically normally distributed, error vector with covariance
matrix $\Omega$. The covariance matrix of $V$, satisfies
$\Sigma_{VV}=(\mathrm{Id} - \mathcal{A})^{-1} \Omega (\mathrm{Id}
-\mathcal{A})^{-T}$, where we use $-T$ to denote the inverse of the transpose
matrix and $\mathrm{Id}$ to denote the identity matrix. A linear structural
equation model corresponds to an ADMG $G$, with the directed edges in $G$
corresponding to the non-zero entries of $\mathcal{A}$ and the bidirected edges
in $G$ to the non-zero off-diagonal entries of $\Omega$. Linear structural
models are used to study the effect of interventions, that is, manipulations of
the generating equations, on the distribution of $V$
\citep{pearl2009causality}. A typical target of inference is the total causal
effect of $x$ on $y$ which is defined as $\tau_{yx}=E\left(y\mid
\mathrm{do}(x=1)\right) - E\left(y\mid \mathrm{do}(x=0)\right)$, where
$E\left(Y\mid \mathrm{do}(x=s)\right)$ is the expected value of $Y$ in
distribution of $V$ we obtain when replacing the generating equation for $x$
with $x=s$. In a linear structural equation model compatible with an ADMG $G$,
$\tau_{yx}$ is equal to the product of the edge coefficients along the directed
paths from $x$ to $y$ in $G$. With this notation we can now provide an example
confirming that the \texttt{DAGitty} algorithm for finding valid conditional
instrumental sets is indeed unsound.

\begin{figure}[t]
  \centering
  \begin{tikzpicture}[xscale=1.1]
    \node (x) at (0,0) {$x$};
    \node (m) at ($(x)+(1.5,0)$) {$m$};
    \node (y) at ($(x)+(3,0)$) {$y$};
    \node (z) at ($(x)+(0,-1.5)$)   {$z$};
    \node (f) at ($(m)+(0,-1.5)$)   {$w$};
    
    \path[arc]   (x) edge    (m);
    \path[arc]   (x) edge     (z);
    \path[arc]   (m) edge    (y);
    \path[arc]   (m) edge    (f);
    \path[arc]   (f) edge    (z);
  \end{tikzpicture}
  \caption{DAG discussed in Example~\ref{example:daggity counter example} as an
  instance where \texttt{DAGitty} returns an invalid conditional instrumental
  set, namely $\{\{z\},\{w\}\}$.}
  \label{figure:daggity bug}
\end{figure}
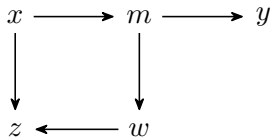

\begin{example}
  \label{example:daggity counter example}
  Consider the DAG $G$ in Figure \ref{figure:daggity bug} and suppose we are
  interested in estimating the causal of effect of $x$ on $y$ using a
  conditional instrument. Applying the graphical criterion by
  \citet{henckel2024graphical} we can see that since $\forb_G(x, y) =
  \{x,m,w,z,y\}$ no valid conditional instrumental set exists. Nonetheless, the
  \texttt{DAGitty} version 3.1 (checked in July 2025) returns the tuple
  $(\{z\},\{w\})$. We now construct an explicit structural equation model
  compatible with $G$ and show that in this model the population level limit of
  the two-stage least-square estimator using the conditional instrument
  $(\{z\},\{w\})$ is unequal to the target causal effect, that is,
  $(\{z\},\{w\})$  is not a valid conditional instrument \citep[see Definition
  1][]{henckel2024graphical}. This shows that the output of \texttt{DAGitty} is
  indeed wrong. Consider the SEM compatible with $G$ with all edge coefficients
  equal to $1$ and all errors standard normally distributed. The covariance
  matrix for the random vector $V=(x,m,w,y,z)$ is:
  $$
  \Sigma_{VV} = \begin{pmatrix}
    1  &  1  &  1  &  1  &  2 \\
    1  &  2  &  2  &  2  &  3 \\
    1  &  2  &  3  &  2  &  4 \\
    1  &  2  &  2  &  3  &  3 \\
    2  &  3  &  4  &  3  &  7
  \end{pmatrix}.
  $$
  Therefore the population level two-stage least squares coefficient $\tau_{yx}^{Z.W}$ is
  \begin{align*}
    \sigma_{yz.w} \sigma_{zx.w}^{-1} 
    &= (\sigma_{yz} - \sigma_{yw} \sigma_{ww}^{-1} \sigma_{wz}) (\sigma_{zx} - \sigma_{zw} \sigma_{ww}^{-1} \sigma_{wx}) \\
    &= \left(3 - 2 \times \frac{1}{3} \times 4\right)\left(2 - 1 \times \frac{1}{3} \times 4\right)^{-1} =\frac{1}{2}.
  \end{align*}
  The total causal effect $\tau_{yx}$, on the other hand, is equal to $1$ as
  there is only one directed path from $x$ to $y$ in $G$ and all edge
  coefficients are by equal to $1$. It follows that $(\{z\},\{w\})$  is not a
  valid conditional instrumental set and the algorithm provided by
  \texttt{DAGitty} is not sound. We note, however, that $z$ is a descendant of $x$ and it is unlikely that practitioners would rely on such an instrument. 
\end{example}

\section{Additional Material for Section \ref{section:cifly:in:action}}
In this section, we provide some more details regarding the algorithm for computing the parent adjustment
distance introduced in Section \ref{section:parent adjustment} and the
corresponding two simulation studies from Section \ref{section:cifly:in:action}.


\begin{figure}[t]
    \begin{lstlisting}[language=ruletable]
EDGES --> <--, ---
SETS X, W
COLORS pass, yield
START --> [pass] AT X
OUTPUT ... [yield]

... [pass]  | ---, --> [pass]  | next not in X and next not in W
... [pass]  | ---, --> [yield] | next not in X and next in W
... [yield] | ---, --> [yield] | next not in X
  \end{lstlisting}
  \caption{CIfly rule table for computing the nodes connected to $X$ via a
  proper possibly directed path that does not contain a node in $W$ in a CPDAG.}
  \label{cifly:tables:parent:aid:forbidden}
\end{figure}

\begin{figure}[t]
    \begin{lstlisting}[language=ruletable]
EDGES --> <--, ---
SETS X, W
COLORS init, causal, non-causal
START ... [init] AT X
OUTPUT ... [non-causal]

... [init]       | ---, --> [causal]     | true
... [init]       | <--      [non-causal] | true
--> [...]        | <--      [non-causal] | next not in X and current in W 
--- [causal]     | ---      [causal]     | next not in X and current not in W 
... [causal]     | -->      [causal]     | next not in X and current not in W 
--- [non-causal] | ---      [non-causal] | next not in X and current not in W 
<-- [non-causal] | ...      [non-causal] | next not in X and current not in W 
... [non-causal] | -->      [non-causal] | next not in X and current not in W 
  \end{lstlisting}
  \caption{CIfly rule table for computing the nodes connected to $X$ via a
  proper definite-status non-causal walk open given $W$ in a CPDAG.}
  \label{cifly:tables:parent:aid:noncausal}
\end{figure}

\subsection{Computing the Parent Adjustment Distance in CIfly}
\label{appendix:gadjid}
The \emph{parent adjustment distance} is a causal measure for the distance
between two DAGs or CPDAGs $G_{\text{guess}}$ and $G_{\text{true}}$ proposed by
\citet{henckel2024graphical} and closely related to the structural intervention
distance \citet{peters2015structural}. For a formal introduction and motivation
of this measure, we refer interested readers to \citet{henckel2024graphical}.
The parent adjustment distance counts the number of node pairs $x$ and $y$ in
the graphs such that (i) Condition \ref{condition: amenability} of Theorem
\ref{theorem:adjustment} is violated relative to $(x,y)$ in $G_{\text{guess}}$
but not in $G_{\text{true}}$, (ii) $y$ is a parent of $x$ in
$G_{\text{guess}}$, which implies a zero causal-effect, but $y$ is a descendant
of $x$ in $G_{\text{true}}$, or (iii) the parents of $x$ in $G_{\text{guess}}$
are a valid adjustment set in $G_{\text{guess}}$ but not a valid adjustment set
relative to $x$ and $y$ in $G_{\text{true}}$. To compute this distance, we need
to check whether the parents of $x$ form a valid adjustment set relative to $x$
and $y$ in both $G_{\text{guess}}$ and $G_{\text{true}}$. To do so efficiently,
\cite{henckel2024adjustment} propose to decide this, given a fixed $x$, for all
$y$ simultaneously in linear time, thus yielding an \emph{overall} time of $O(p
\cdot (p + m))$. To achieve this they reformulate the CPDAG adjustment
criterion as follows.

\begin{theorem}[\citet{henckel2024adjustment}]
  Let $X$, $\{y\}$ and $W$ be pairwise disjoint node sets in a CPDAG $G$.
  Then $W$ is a valid adjustment set relative to $(X,y)$ in $G$ if, and only
  if,
  \begin{enumerate}
    \item no proper possibly directed walk that does not begin with a directed
      edge out of $X$ reaches $y$,
    \item no proper possibly directed walk that contains a node in $Z$ reaches $y$, and
    \item no proper definite-status non-causal walk that is not blocked by $Z$ reaches $y$.
  \end{enumerate}
  \label{theorem:parent:aid:conditions}
\end{theorem}

Hence, given a graph $G$ and a treatment $x$, we can find all $y$ that violate
Condition 1, 2 and 3, respectively, using three reachability calls. We show how
to express these reachability algorithms in CIfly below. Note that Conditions 2
and 3 in Theorem \ref{theorem:parent:aid:conditions} are \emph{not} one-to-one
rephrasings of Conditions 2 and 3 from Theorem~\ref{theorem:adjustment}. 

In Figure~\ref{cifly:code:parent:aid}, we give the Python code for computing
the parent adjustment distance between two CPDAGs. It relies on rule tables for
finding all $y$ that violate Condition 1, 2 and 3 of
Theorem~\ref{theorem:parent:aid:conditions}. Concretely, we can re-use the
left table in Figure~\ref{figure:adjustment simple} to handle the
first condition. For the second and third condition, the tables in
Figure~\ref{cifly:tables:parent:aid:forbidden}
and~\ref{cifly:tables:parent:aid:noncausal} are used. 

\begin{figure}[p]
  \begin{lstlisting}[language=Python, style=codeStyle]
import ciflypy as cifly

def poss_desc(g, x):
    return set(cifly.reach(g, {"X": x}, poss_desc_table))

def not_amenable(g, x):
    return set(cifly.reach(g, {"X": x}, not_amenable_table))

def forbidden(g, x, W):
    return set(cifly.reach(g, {"X": x, "W": W}, forb_path_table))

def non_causal(g, x, W):
    return set(cifly.reach(g, {"X": x, "W": W}, non_causal_table))

def not_amen_not_adjust(g, x, W):
    nam = not_amenable(g, x)
    return nam, nam.union(forbidden(g, x, W)).union(non_causal(g, x, W))

# edges_true and edges_guess represent the true and guessed graphs; as CIfly graphs are represented by edgelists, we also pass the number of nodes p
def parent_aid(p, edges_true, edges_guess):
    parents = [[] for _ in range(p)]
    for u, v in edges_guess["-->"]:
        parents[v].append(u)
    # for extra performance, pre-parse edgelists into CIfly graphs
    g_guess_parsed = cifly.Graph(edges_guess, poss_desc_table)
    g_true_parsed = cifly.Graph(edges_true, poss_desc_table)

    mistakes = 0
    for x in range(p):
        pa = set(parents[x])
        nam_guess = not_amenable(g_guess_parsed, x)
        desc_true = poss_desc(g_true_parsed, x)
        nam_true, nad_true = not_amen_not_adjust(g_true_parsed, x, pa)
        for y in range(p):
            if y == x:
                continue
            if y in pa:
                if y in desc_true:
                    mistakes += 1
            elif y in nam_guess:
                if y not in nam_true:
                    mistakes += 1
            else:
                if y in nad_true:
                    mistakes += 1
    return mistakes
  \end{lstlisting}
\caption{Python code for computing the parent adjustment distance for CPDAGs.}
  \label{cifly:code:parent:aid}
\end{figure}

The algorithm proceeds by going through all nodes $x$ to check for mistakes
between the two given graphs. First, all nodes $y$ are computed for which no
adjustment set relative to $(x, y)$ exists. These are the $y$ that are
\emph{not amenable} in $G_{\text{guess}}$. For those where this is not the
case, we also compute the parents of $y$ in $G_{\text{guess}}$ and store them.
Afterwards, for $G_{\text{true}}$ it is checked (i) which $y$ are descendants
of $x$, that is for which $y$ the causal effect is zero, (ii) which $y$ are not
amenable and (iii) for which $y$ the parents of $x$ in $G_{\text{guess}}$ are
not a valid adjustment set in $G_{\text{true}}$. Finally, for each $y$, it is
evaluated whether there is a discrepancy between the graphs. For details of the
underlying logic, we refer to~\citep{henckel2024adjustment}. 

The code given in Figure \ref{cifly:code:parent:aid} is exactly the code we use
in the simulation study, where it is faster than \texttt{gadjid} on dense
graphs. The only optimization that is necessary is to parse  the graphs
represented by the edgelists \texttt{edges\_guess} and \texttt{edges\_true}
before the for-loop. This is due to the fact that this parsing can cause
considerable overhead in the \texttt{cifly.reach} calls. After all,
\texttt{reach} runs in linear time in the input size, that is, the overall run
time is proportional to reading the input. CIfly offers the \texttt{Graph}
constructor for this, to which an edge list in CIfly format has to be passed
and a rule table with the correct signature of the graph (the parsed graph will
be compatible with all rule tables with the same line \texttt{EDGES \dots}). On
\url{cifly.pages.dev}, we also provide the first implementation of the parent
adjustment distance in the R programming language. 

\subsection{Simulation Setup}
\label{appendix:adjustment simulation}
For the simulations, we generate CPDAGs in the following manner. We generate an
undirected Erdős-Renyi random graph with $p$ nodes and average degree $d$.
Afterwards, we orient the edges to directed ones according to a uniformly drawn
linear ordering of the nodes. Finally, we compute the CPDAG of whose
equivalence class this DAG is a member. 

For the experiments regarding the CPDAG adjustment criterion, we generate
CPDAGs with $d=4$. For $p$ being equal to $100$, $200$, $300$, $400$ and $500$,
we generate $20$ independent CPDAGs and average the run times. In case of the
parent adjustment distance, we generate CPDAGs with $d=4$ (sparse graphs) and
$d = p / 10$ (dense graphs). Again, we generate $20$ independent CPDAGs for $p$
being $100$, $200$, $300$, $400$ and $500$, and average the resulting
run times. In both simulations, we check that all algorithms output the same
results.

We compare against \texttt{pcalg} version 2.7.12, DAGitty version
\texttt{0.3.4} and \texttt{gadjid} version 0.1.0 on a single thread of a
Intel\textregistered{} Core\textsuperscript{TM} i5-8350 CPU. For gadjid, the
number of threads is controlled via the environment variable
\texttt{RAYON\_NUM\_THREADS}.

\end{document}